\documentclass[letterpaper]{article} 
\usepackage{aaai25}  
\usepackage{times}  
\usepackage{helvet}  
\usepackage{courier}  
\usepackage[hyphens]{url}  
\usepackage{graphicx} 
\usepackage{natbib}  
\usepackage{caption} 
\usepackage{algorithm}
\usepackage{algorithmic}

\usepackage{newfloat}
\usepackage{listings}
\DeclareCaptionStyle{ruled}{labelfont=normalfont,labelsep=colon,strut=off} 
\floatstyle{ruled}
\newfloat{listing}{tb}{lst}{}
\floatname{listing}{Listing}
 \nocopyright

\usepackage{amssymb}
\usepackage{mathtools}
\usepackage{amsthm}
\usepackage{algorithmic}
\usepackage{algorithm}
\usepackage{multirow}
\usepackage{subfig}
\usepackage{placeins}
\usepackage{makecell}
\usepackage[figuresright]{rotating}
\usepackage{tabularx}
\usepackage[capitalize,noabbrev]{cleveref}

\newtheorem{theorem}{Theorem}

\newtheorem{definition}{Definition}
\newtheorem{lemma}{Lemma}
\allowdisplaybreaks

\title{k-HyperEdge Medoids for Clustering Ensemble}
\author{
    Feijiang Li, Jieting Wang, Liuya zhang, Yuhua Qian, Shuai jin, Tao Yan, Liang Du
}
\affiliations{
    Institute of Big Data Science and Industry, Shanxi University\\
}

\usepackage{bibentry}

\begin{document}

\maketitle

\begin{abstract}
Clustering ensemble has been a popular research topic in data science due to its ability to improve the robustness of the single clustering method. Many clustering ensemble methods have been proposed, most of which can be categorized into clustering-view and sample-view methods. The clustering-view method is generally efficient, but it could be affected by the unreliability that existed in base clustering results. The sample-view method shows good performance, while the construction of the pairwise sample relation is time-consuming. In this paper, the clustering ensemble is formulated as a k-HyperEdge Medoids discovery problem and a clustering ensemble method based on k-HyperEdge Medoids that considers the characteristics of the above two types of clustering ensemble methods is proposed. In the method, a set of hyperedges is selected from the clustering view efficiently, then the hyperedges are diffused and adjusted from the sample view guided by a hyperedge loss function to construct an effective k-HyperEdge Medoid set. The loss function is mainly reduced by assigning samples to the hyperedge with the highest degree of belonging. Theoretical analyses show that the solution can approximate the optimal, the assignment method can gradually reduce the loss function, and the estimation of the belonging degree is statistically reasonable. Experiments on artificial data show the working mechanism of the proposed method. The convergence of the method is verified by experimental analysis of twenty data sets. The effectiveness and efficiency of the proposed method are also verified on these data, with nine representative clustering ensemble algorithms as reference.
\end{abstract}

%
\section{Introduction}

Data clustering is an interesting technique to discover the group structure inherent in a set of data samples without labels \cite{vega2011survey}. The general objective of data clustering is discovering a set of clusters in which the samples in the same cluster share some kind of high similarity while the samples in different clusters share high diversity. With the complexity of the data development \cite{li2020}, many traditional clustering methods suffer from the robust problem. To handle this, the clustering ensemble technique that combines multiple diverse clustering results has been illustrated as an effective strategy to improve the robustness of a single clustering result.

The clustering ensemble problem is first proposed by Strehl and Ghosh \cite{Strehl}, which is described as combining multiple clustering results of a set of objects into a single consolidated clustering that shares most information with the ensemble base without accessing the features of the data samples or the algorithms that produce these clustering results. Following this description, many methods have been proposed to solve the clustering ensemble problem. Most of the existing clustering ensemble methods can be roughly categorized into clustering-view and sample-view.

As for the clustering-view methods, their main idea is to generate a clustering result that is similar to the clustering results in the ensemble base. To generate such a clustering result, a simple strategy is to first align the labels of the different clustering results and then merge them using a fusion strategy, such as voting. Another strategy is to give a clustering similarity measure and then use a dynamic programming method to produce the integrated result. Generally, due to that the number of clustering results is much less than that of samples, most of the clustering-view methods are efficient. However, the unreliability that existed in the base clustering results brings a negative effect on the ensemble performance \cite{ZHOU2022171, wang2022generalization}.

As for the sample-view methods, their main idea is to generate a clustering result that the samples in the same cluster share high similarity. There are two issues for this type of method, one is how to evaluate the similarity between samples and the other one is how to discover such a clustering result. As for the first issue, the co-association relation that evaluates the times that two samples appear in the same cluster is one of the most utilized sample pair similarity measures \cite{li2021}. As for the second issue, many partition methods have been proposed, in which the clustering strategies \cite{liu2017spectral,cheng2024deep}, graph partition strategies \cite{Strehl, li2023fuzzy}, or heuristic strategies \cite{li2019clustering} are utilized. This type of method has shown excellent effectiveness in handling the clustering ensemble problem. However, calculating the pairwise sample relationships and optimizing a clustering loss is often time-consuming.

With consideration of the characteristics of the above two types of clustering ensemble methods, in this paper, a method based on k-HyperEdge Medoids is developed, which integrates clustering-view and sample-view. First, the description of k-HyperEdge Medoids is given. For a hypergraph, the k-HyperEdge Medoids are $k$ nonoverlapping hyperedges that each hyperedge in the hypergraph could find a similar hyperedge in the k-HyperEdge Medoids. Then the clustering ensemble problem is formulated as a k-HyperEdge Medoids discovering problem. We propose a method to solve the problem and conduct theoretical and experimental analysis about the method.

The main work of this paper in terms of methods, theory, and experiments includes:

\textbf{Method.} A k-HyperEdge Medoids discovery method is proposed \footnote{Demo code at https://github.com/FeijiangLi/Code-k-HyperEdge-Medoids-for-Clustering-Ensemble-AAAI-24}. This method mainly contains three steps, which are k-HyperEdge initialization, k-HyperEdge diffusion, and k-HyperEdge adjustion. The k-HyperEdge initialization step is from the clustering view, which utilizes the k-mediods algorithm to discover a set of sub-hyperedges mediods. The k-medoids algorithm is suitable for arbitrary distance measurement and robust to noise \cite{Tiwari2020}. The k-HyperEdge diffusion and k-HyperEdge adjusting improve the quality of the initial hyperedges methods from the sample view. The improvement is realized by updating the sample assignment guided by a hyperedge loss function based on cluster quality. The loss function is mainly reduced by assigning samples to the hyperedge with the highest degree of belonging.

\textbf{Theory.} Three main issues in the proposed method are theoretically analyzed. Firstly, given the initial hyperedge set, the solution has been proven to be able to approximate to the optimal, indicating its ability to guide model learning. Secondly, assigning samples to the hyperedge with the highest belonging degree has been proven to be able to reduce the value of the loss function. Finally, the method for estimating the belonging degree has been proven to be statistically reasonable.

\textbf{Experiment.} The proposed method is experimentally analyzed from three aspects. The working mechanism of the proposed method is illustrated on artificial data. The effectiveness of the method in solving the clustering ensemble problem is illustrated on twenty real-world data sets compared with the other nine representative clustering ensemble methods. The time complexity of the proposed method is analysed and the time costs of the compared methods on real data sets are reported. The results show the effectiveness and efficiency of the proposed method.

\section{Preliminaries and Related Work}

\subsection{Clustering Ensemble}

As for clustering ensemble, what we can obtain is a set of clustering results $\Pi=\{\pi^1, \pi^2,\ldots, \pi^l\}$ on a data set $\mathcal{X}=\{x_1, x_2,\ldots, x_n\}$, where $l$ is number of base clustering result and $n$ is the number of data point. The $j$th clustering result contains $k_j$ clusters $\pi^j=\{c^j_1, c^j_2, \ldots, c^j_{k_j}\}$ and $\mathcal{X}=\bigcup_{i=1}^{k_j} c^j_i$. The cluster label of $x_i$ in $\pi^j$ is $\pi^j(x_i)$. For $\Pi$, there exists $n_c=\sum_{i=1}^{h}k_i$ base clusters.
The task of the clustering ensemble is to discover a clustering result $\pi^*=\{c^*_1, c^*_2,\ldots c^*_k\}$ that shares the most information with $\Pi$, where $k$ is the excepted number of clusters. In the original description of the clustering ensemble, the generation of $\pi^*$ does not access the features or algorithms that determine these clustering results \cite{Strehl}. This paper follows this description.

Many clustering ensemble methods have been proposed, accompanied by various classification schemes for these methods \cite{vega2011survey}. From the view of processing, most of the clustering ensemble methods can be roughly categorized into clustering-view and sample-view.

\subsection{Clustering-view Clustering Ensemble}

The clustering ensemble methods from the clustering-view are mainly guided by the objective that the integrated clustering should be similar to the base clustering results. Then, the objective can be expressed as follows,
\begin{align}\label{obj1}
\pi^*=\mathop{\arg \max}_{\pi }\sum_{\pi' \in \Pi} \text{sim}(\pi, \pi'),
\end{align}
where $\text{sim}(\pi, \pi')$ is a similarity between $\pi$ and $\pi'$.

The optimization of the objective is an NP-hard problem, most of the proposed methods generally obtain an approximate solution. One direct approach is to first align the cluster labels and then merge them using a voting strategy \cite{zhou2006clusterer}. Ayad et al. \cite{ayad2008cumulative, ayad2010on} introduced an idea of cumulative voting that utilizes a probabilistic mapping method for cluster label alignment, and computed empirical probability summarizing the ensemble. Carpineto et al. \cite{carpineto2012consensus} proposed a probabilistic rand index and cast clustering ensemble as an optimization problem of the PRI between a target result and the ensemble base. Li et al. \cite{Li2017} proposed a Dempster-Shafer evidence theory-based clustering ensemble method that substitutes voting with evidence theory. Huang et al. \cite{HUANG2023109255} proposed an ensemble hierarchical clustering algorithm based on merits at the clustering level, the main idea of which is to re-cluster the selected clusters to generate a consensus clustering. Guilbert et al. \cite{Guilbert2022} proposed an anchored constrained clustering ensemble method that constructs an allocation matrix and uses an integer linear programming model to find the consensus partition.

Generally, these methods are simple and efficient. However, the performance of these methods relies on the reliability of the underlying clustering ensemble. Due to the diversity among the base clustering results, low-reliability results exist in the ensemble inevitably, which could affect the performance of many clustering-view clustering ensemble methods \cite{wang2023rss}. 

\subsection{Sample-view Clustering Ensemble}

The clustering ensemble methods from the sample view are mainly guided by the objective that the intra-cluster samples in the integrated clustering should be similar to each other. This objective is equivalent to finding a clustering result with high cluster quality, where cluster quality is measured by the consistency among intra-cluster samples. The objective can be expressed as follows
\begin{equation}\label{obj2}
\pi^*=\mathop{\arg\max}_{\pi}\sum_{c_i \in \pi} Q(c_i),
\end{equation}
where $Q(c_i)$ indicates the quality of $c_i$. From the sample view, $Q(c_i)$ is generally evaluated by the consistency between samples in $c_i$.

This type of methods are generally realized by optimizing a clustering loss function. Fred et al. \cite{Fred} first proposed the concept of the co-association matrix (CA matrix) and proposed an evidence accumulation method that utilizes the idea of hierarchical clustering. The spectral ensemble clustering method proposed by Liu utilizes the spectral clustering method to optimize the clustering loss \cite{liu2017spectral}. Iam-On proposed a type of link-based approach to refine the cluster-based matrix, and k-means, PAM, and graph partitioning algorithms are used to generate the final result. Graph partitioning techniques have been used in many clustering ensemble methods to optimize the clustering loss, such as CSPA, HGPA, MCLA \cite{Strehl}, HBGF \cite{fern2004solving}, PTGP \cite{huang2016robust}, U-SENC \cite{Huang2020},  SCCBG \cite{zhou2020p295}, etc. Hao et al. \cite{Hao10173506} proposed a clustering ensemble method with attentional representation that jointly optimizes the clustering loss and the reconstruction loss. In addition, many enriching methods are proposed to improve the relation between samples, such as the probability trajectory proposed by Huang et al. \cite{huang2016robust}, the propagating cluster-wise similarities with meta-cluster proposed by Huang et al. \cite{huang2021}, co-association matrix self-enhancement proposed by Jia et al. \cite{Jia2023}, dense representation proposed by Zhou et al. \cite{Zhou2019Ensemble}, stable clustering ensemble method based on evidence theory proposed by Fu et al. \cite{Fu2022}, etc.

This type of method has shown excellent ensemble performance. However, for such methods, the construction of the co-association relationship is time-consuming, so as with many enriching methods, the time cost of which is generally at least $O(n^2)$, where $n$ is the number of samples. Although some methods utilize representative point selection or downsampling strategies for acceleration, this affects the stability and effectiveness of the methods.


In this paper, considering the characteristics of the above two types of clustering ensemble methods, we propose a clustering ensemble method based on k-HyperEdge Medoids. 

\section{k-HyperEdge Medoids}

It is natural to represent a set of clustering results as a hypergraph \cite{ZHOU2022171}. In this section, we introduce the k-HyperEdge Medoids.

A hypergraph $\mathcal{G}$ can be represented as $\mathcal{G}= \{\mathcal{V}, \mathcal{E}, \mathcal{W}\}$, where $\mathcal{V}$ is the node set, $\mathcal{E}$ is the hyperedge set, and $\mathcal{W}$ is the set of hyperedge weight. Mainly different from the traditional edge in the graph, the hyperedge links a set of nodes in the graph that can represent complex multivariate relationships more flexibly. For example, in a hypergraph with a node set $\mathcal{V}=\{x_1, x_2, x_3, x_4, x_5\}$, a hyperedge $e$ could connect multiple nodes, such as $e=\{x_1, x_2, x_5\}$, representing a common relationship or connection among $x_1$, $x_2$, and $x_5$. 

Given a set of clustering results $\Pi=\{\pi^1, \pi^2,\ldots, \pi^l\}$ on a data set $\mathcal{X}=\{x_1, x_2,\ldots, x_n\}$, from cluster view, $\Pi$ can be expressed as $\Pi_c=\{c^1_1, c^1_2,\ldots, c^1_{k_1},\ldots, c^l_1, c^l_2,\ldots, c^l_{k_l}\}$. For convenience, we denote
\begin{align}\label{pic}
\Pi_c=\{c_1, c_2, \ldots, c_{n_c}\}.
\end{align}
The hypergraph representation of $\Pi_c$ is $\mathcal{G}(\Pi_c)= \{\mathcal{V}(\Pi_c), \mathcal{E}(\Pi_c), \mathcal{W}(\Pi_c)\}$, where $\mathcal{V}$ is the set of data points $\mathcal{V}(\Pi_c)=\mathcal{X}$, $\mathcal{E}$ is the set of clusters $\mathcal{E}(\Pi_c)= \Pi_c$, and $\mathcal{W}$ is the weight of each cluster. The cluster weight can be quantified by cluster quality evaluation. In this paper, the cluster weight is not a focus and is set equally.

\begin{definition}\label{defkhm}
\textbf{(k-HyperEdge Medoids)} For a hypergraph $\mathcal{G}= \{\mathcal{V}, \mathcal{E}, \mathcal{W}\}$, the k-HyperEdge Medoids $\mathcal{E}_m(\mathcal{G})$ are $k$ hyperedges
\begin{align}\label{khm}
&\mathcal{E}_m(\mathcal{G})=\mathop{\arg \max}_{\mathcal{E}'}\sum_{e \in \mathcal{E}} \max_{e' \in \mathcal{E}} \text{sim}(e, e'),\\ \notag
\text{s.t. }& \bigcup_{e\in \mathcal{E}_m}=\mathcal{V}, \quad \forall e_i, e_j \in \mathcal{E}_m, i\neq j, e_i\cap e_j=\emptyset.
\end{align}
\end{definition}

From Definition \ref{defkhm}, it is known that the k-HyperEdge Medoids are k representative hyperedges that maximize overall similarity with all hyperedges in the hypergraph. Additionally, the k-HyperEdge Medoids are non-overlapping and collectively cover all nodes. From the view of clustering ensemble, each cluster can be treated as a hyperedge, and multiple clusters form a hypergraph. Then the k-HyperEdge Medoids of $\mathcal{G}(\Pi_c)$ can be treated as a clustering ensemble result of $\Pi_c$, and the clustering ensemble problem can be formulated as a k-HyperEdge Medoids discovering problem.

\section {A k-HyperEdge Medoids Discovery Method}

In this section, a heuristic k-HyperEdge Medoids discovery method is proposed to approximatively solve Formula (\ref{khm}). The method mainly contains three steps, which are k-hyperedge initialization, k-hyperedge diffusion, and k-hyperedge adjustion. The main framework of the k-HyperEdge Medoids discovers method is shown as Figure \ref{frame}. Then processes of each step are introduced.

\begin{figure}[t] %
\begin{center}
\includegraphics[width=0.45\textwidth]{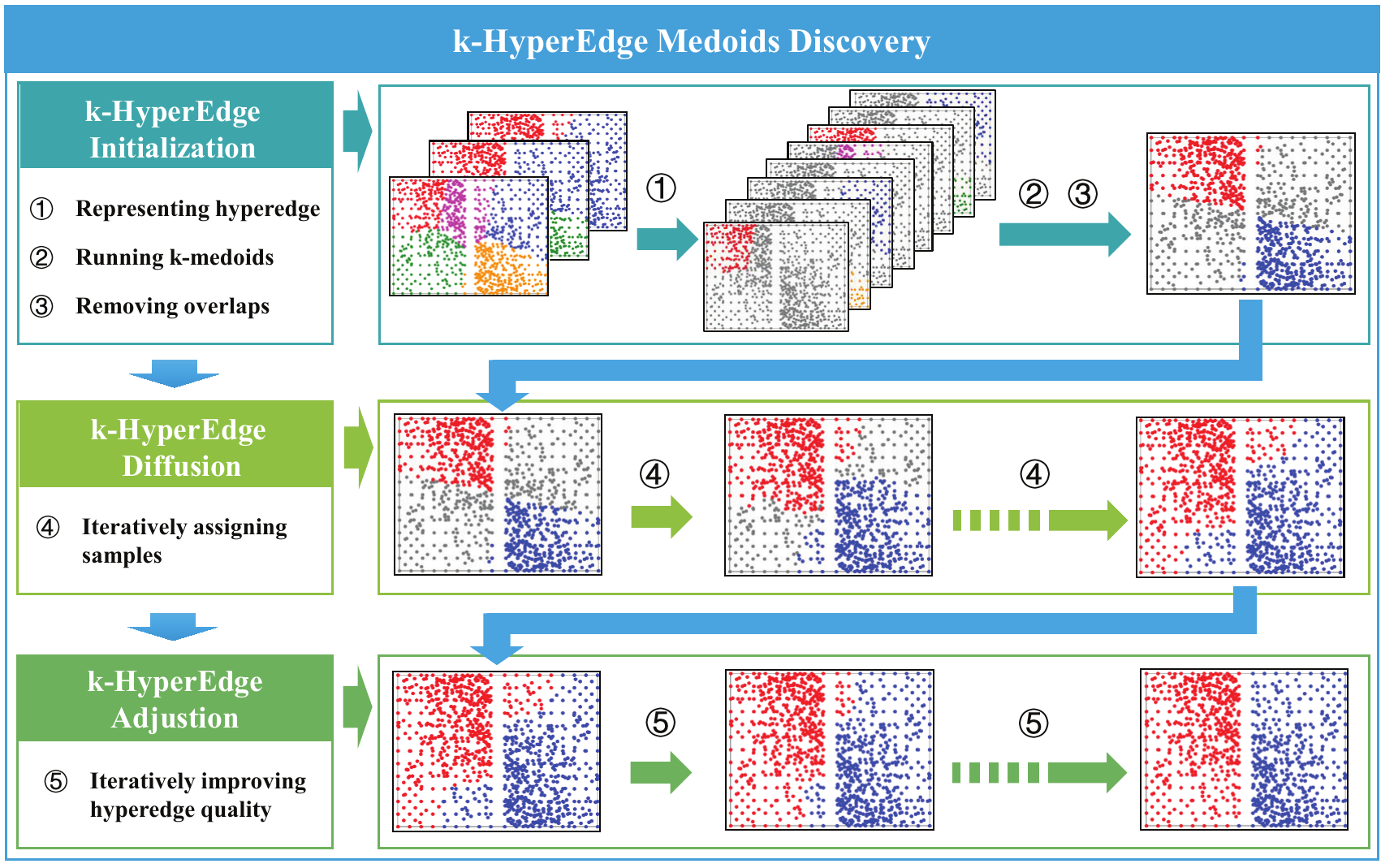}
\caption{The framework of a k-HyperEdge Medoids discovery method}
\label{frame}
\end{center}
\end{figure}

\subsection{k-HyperEdge Initialization}

The objective of the k-hyperedge initialization step is finding a set of $k$ hyperedges from $\mathcal{E}(\Pi_c)=\{c_1, c_2, \ldots, c_{n_c}\}$, i.e. $\mathcal{E}i=\{ei_1, ei_2, \ldots ei_k\}\subset \mathcal{E}$, that minimizes the following loss:
\begin{equation}\label{h_i}
\mathcal{L}_{\text{ini}}(\mathcal{E}i)=\sum_{i=1}^{n_c} \min_{ei\in \mathcal{E}i} \left(1-\text{sim}(ei, c_i)\right),
\end{equation}
where $sim(ei, c_i)$ is a similarity measure between $ei$ and the $i$th hyperedge $c_i$. To prevent the discovery of an excessively large hyperedge, in this step, we use the following similarity measure
\begin{equation}\label{jc}
\text{sim}(ei, c_i)=\frac{|ei \cap c_i|}{|c_i|}.
\end{equation}

As shown by Equation (\ref{h_i}), what we want to find is $k$ hyperedges, in which each hyperedge in $\mathcal{E}(\Pi_c)$ could find a similar representation. Equation (\ref{h_i}) is equivalent to the loss function of the k-medoids algorithm. Then the k-medoids algorithm can be directly utilized to solve (\ref{h_i}). Meanwhile, the k-medoids algorithm has several advantages in this step. Firstly, the medoids are real points, which are interpretable. Secondly, k-medoids are not sensitive to noise, which could mitigate the negative effects of unreliability clusters. Thirdly, the distance or consistency function in k-medoids can be arbitrary, which makes it possible to adopt Equation (\ref{jc}) \cite{Tiwari2020}.

After run k-medoids algorithm on $\mathcal{E}(\Pi_c)$, $k$ hyperedges are obtained, which is noted as $\mathcal{E}i=\{ei_1, ei_2, \ldots ei_k\}$. Due to the selected hyperedges could come from different clustering results, which may not satisfy $\bigcup_{ei \in \mathcal{E}i} ei=\mathcal{V}$ and $\forall ei_i, ei_j \in \mathcal{E}i, i\neq j, ei_i\cap ei_j=\emptyset$. To handle the first condition, we remove the intersection regions between different hyperedges and regenerate $\mathcal{E}i=\{ei_1, ei_2, \ldots ei_k\}$. To handle the second condition, we conduct the following hyperedge diffusion step. The processes of the k-HyperEdge initialization step are shown as Algorithm \ref{KHI} in Appendix A.1.

\subsection{k-HyperEdge Diffusion}

The purpose of the k-HyperEdge diffusion step is to make each sample in $\mathcal{X}$ belong to a hyperedge. This step is realized by iteratively assigning a sample to the hyperedge which it is most likely to belong. Firstly, in this step, we initialize $\mathcal{E}d = \mathcal{E}i$. In each iteration, a set of samples that with high belonging confidence will be reassigned to the k initialization hyperedges $\mathcal{E}d=\{ed_1, ed_2, \ldots ed_k\}$. With the progress of iteration, the number of samples to be assigned gradually increases. The diffusion step stops when the number of the assigned sample is equal to the number of data points. The calculation of the sample confidence and the number of samples to be assigned are then introduced.

Firstly, given $\mathcal{E}(\Pi_c)$, the belonging degree of sample $x_i$ to a hyperedge $ed$ is defined as the average similarity between $ed$ and the hyperedges that $x_i$ belongs to in $\mathcal{E}(\Pi_c)$. The equation is
\begin{equation}\label{bx}
b\left(x, ed\right)= \sum_{x\in c_x} |c_x \cap ed|,
\end{equation}
where $c_x\in\mathcal{E}(\Pi_c)$.

Based on Equation (\ref{bx}), the belongings of a sample to all the hyperedges in $\mathcal{E}d$ can be obtained. Then, the hyperedge to which sample $x_i$ belongs is obtained by
\begin{equation}\label{hx}
ed^*(x_i)= \mathop{\arg\max}_{ed \in \mathcal{E}d} b(x_i, ed).
\end{equation}

The sample $x_i$ will be assigned to $ed^*$ with a confidence degree of
\begin{align}\label{mx}
m(x_i) = b\left(x_i, ed^*(x_i)\right) - \max_{\substack{ed_j\neq ed^*(x_i), \\ ed_j \in \mathcal{E}d} } b\left(x_i, ed_j\right),
\end{align}
which is the difference between the maximum and the second maximum degree of belonging.

With Equation (\ref{mx}), the confidence of all the samples in $\mathcal{X}$ can be obtained, and the samples with high confidence will be selected and re-assigned to the hyperedge set in hand. Assuming $\mathcal{E}d$ is the present hyperedge set, the number of samples in $\mathcal{E}d$ is
\begin{align}\label{n_a}
n_a=\left|\bigcup_{ed \in \mathcal{E}d} ed \right|.
\end{align}
With the overall consideration of $n_a$ and $n$, the number of samples $n_s$ to be selected is calculated as
\begin{equation}\label{na}
n_s=\min \{( n_a+ \max \{\lceil\frac{n-n_a}{k} \rceil ,\lceil \sqrt{n}\rceil\} ), n \}.
\end{equation}

For an edge $ed$, the diffusion is realized by
\begin{align}\label{hdif}
ed =\{x | ed^*(x)=ed, m(x) \geq m_{(n_s)}(\mathcal{X}) \},
\end{align}
where $m_{(n_s)}(\mathcal{X})$ is the $n_s$-th largest element in the set $\{m(x_1)$, $m(x_2)$, $\ldots$, $m(x_n)\}$.

Repeating the above diffusion process, the number of samples in $\mathcal{E}d$ will increase. When $n_a = n$, the k-HyperEdge diffusion process stops, and the obtained hypergedge set after this step is noted as $\mathcal{E}d=\{ed_1, ed_2, \ldots, ed_k\}$, which satisfies $\bigcup_{ed\in \mathcal{E}d}=\mathcal{V}$ and $\forall ed_i, ed_j \in \mathcal{E}d, i\neq j, ed_i\cap ed_j=\emptyset$. The processes of the k-HyperEdge diffusion step are shown as Algorithm \ref{KHD} in Appendix A.1. Further, the quality of $\mathcal{E}d$ is improved by the following k-HyperEdge adjustion step.

\subsection{k-HyperEdge Adjustion}

After assigning all the samples in $\mathcal{X}$, we further improve the quality of the hyperedge set from the sample view according to Equation (\ref{obj2}). Assuming the hyperedge set is $\mathcal{E}a=\{ea_1, ea_2,\ldots, ea_k\}$, the quality of $ea_i$ is
\begin{align}\label{qed}
Q(ea_i)=\sum_{x \in ea_i} b(x, ea_i).
\end{align}

The loss function of the k-HyperEdge adjustion step is
\begin{align}\label{lka}
\mathcal{L}_{\text{adj}}(\mathcal{E}a) =& \sum_{i=1}^{k} \left( n |ea_i| - Q(ea_i) \right) \\ \notag
=&\sum_{i=1}^{n} \left(n- b(x_i, ea^*(x_i))\right).
\end{align}

The loss value can be iteratively reduced by updating $\mathcal{E}a$ until convergence.

Assuming the hyperedge set in the $t$-th iteration is $\mathcal{E}a^t=\{{ea^t}_1, {ea^t}_2, \ldots, {ea^t}_k\}$, the belongings of a sample $x_i$ to a hyperedge ${ea^t}_j$ in $\mathcal{E}a^t$ can be calculated by Equation (\ref{bx}), noted as $b(x_i, {ea^t}_j)$. Then, the hyperedge that $x_i$ belongs to in the $(t+1)$-th iteration is
\begin{align}\label{btad}
ea^{(t+1)*}(x_i)= \mathop{\arg\max}_{ea^t \in \mathcal{E}a^t} b(x_i, ea^t).
\end{align}

The $i$-th hyperedge is updated by
\begin{align}\label{upedge}
{ea^{(t+1)}}_i=\{x | ea^{(t+1)*}(x)={ea^t}_i, x \in \mathcal{X} \}.
\end{align}
The $k$ updated hyperedges constitute the hyperedge set $\mathcal{E}a^{(t+1)}$. The adjustion will stops when $\mathcal{E}a^{(t+1)} = \mathcal{E}a^{t}$. The processes of the k-HyperEdge adjustion step are shown as Algorithm \ref{KHA} in Appendix A.1.

The final hyperedge set is the discovered k-HyperEdge Medoids and noted as $\mathcal{E}_m=\{e_1, e_2, \ldots, e_k\}$. The final corresponding clustering ensemble result is
\begin{align}\label{piedge}
\pi^*=\{c_1=e_1, c_2=e_2, \ldots, c_k=e_k\}.
\end{align}
We note the above clustering ensemble method based on the construction of k-HyperEdge Medoids as CEHM. The processes of CECH is shown as Algorithm \ref{CEHM} in Appendix A.1.

\section{Theoretical Analysis}
In this section, we give three theorems about the key element in the method to show the rationality.

First, given a g-approximate initial hyperedge set $\mathcal{E}a^0$, we prove that the solution can approximate to the optimal, which is reflected by the following theorem.

\begin{theorem} \label{the_bound}
let $\mathcal{E}a^0=\{{ea^0}_1, {ea^0}_2,\ldots, {ea^0}_k\}$ be a set of hyperedges produced by a g-approximate algorithm  
based on a base cluster set $\Pi_c$, and Let 
\begin{align}
dist(ea_i, ea_j) = n- |ea_i \cap ea_j|,
\end{align}
\begin{align}
\mathcal{\phi}(ea_i; ea_j) & = \sum_{x\in ea_j} \left(n- b(x, ea_i) \right) \\ \notag
& = \sum_{x\in ea_j} \left(n- \frac{1}{l}\sum_{x\in c_x} |c_x \cap ea_i| \right),
\end{align}
then for any hyperedge set, denoted by $\mathcal{E}a=\{{ea}_1, {ea}_2,\ldots, {ea}_k\}$,  we have for any $ea_i$, there exists ${ed^0}_i$ satisfies $dist(ea_i, {ea^0}_i) \le \frac{(g+1) \phi (ea_i; *)}{|ea_i|}$.
\end{theorem}

Secondly, we prove that assigning samples to the hyperedge with the highest belonging degree can reduce the value of the loss function.

\begin{theorem} \label{the_re}
Assuming the optimal hyperedge of sample $x$ is the $i$-th hyperedge: $ea^*(x)=ea_i$; assuming a 
hyperedge sets $\mathcal{E}a^1=\{{ea^1}_1, {ea^1}_2,\ldots, {ea^1}_k\}$ and $x$ is in the $i$-th hyperedge in $\mathcal{E}a^1$: $x\in {ea^1}_i$; assuming another hyperedge set $\mathcal{E}a^2$ based on $\mathcal{E}a^1$, and in $\mathcal{E}a^2$, $x$ is moved from the $i$-th hyperedge into the $j$-th hyperedge: for $\forall p \in \{1,2,..,k\}/\{i,j\}$, ${ea^2}_i={ea^1}_i/x$, ${ea^2}_j={ea^1}_j\cup x$. Then, we have $\mathcal{L}_{\text{adj}}(\mathcal{E}a^1)\leq \mathcal{L}_{\text{adj}}(\mathcal{E}a^2)$.
\end{theorem}


Thirdly, we prove that using neighbors to estimate the belonging degree is statistically reasonable. To evaluate the similarity between $x$ and $ea$, Equation (\ref{bx}) uses the number of samples in the intersection set of the cluster $c_x$ and $ea$. Indeed, to estimate the posterior probability $\eta(X)=\mathbb{P}(x\in ea|X=x)$, the essence of Equation (\ref{bx}) is to use the empirical conditional posterior probability

\begin{equation}
\hat{\eta}_N(x)= \frac{1}{N_ {c_x} }\sum_{i:X_i\in { c_x}} \mathbb{I}\{X_i\in ea\},
\end{equation}
where $ N_ {c_x}$ is the number of samples in $c_x$. 

To evaluate of the estimator $\hat{\eta}_N(x)$, we explore the upper bound of the estimation error $\mathbb{E}|\hat{\eta}_N(X)-\eta(X)|$, which is reflected by the following theorem.

\begin{theorem}
\label{G-E}
Consider an estimator $\hat{\eta}_N(X)$ of $\eta(X)$ as defined above. 
Suppose that the second moment of $\eta(X)$ of $ea$ exists: 
$\mathbb{E}_{x\in ea}  \eta(X) ^2 <+\infty $ and the difference
in probability density between set $ea$ and the whole set $\mathcal{X}$ is bounded:
that is, there exists $C$, such that 
$\mathbb{E}_{x\in ea}
\left(\frac{\textbf{\textit{m}}(x|x\in\mathcal{X})}{\textbf{\textit{m}}(x|x\in ea)}\right)^2<C$.
Then $\mathbb{E}|\hat{\eta}_N(X)-\eta(X)|\rightarrow 0$ if
\begin{align}
\quad &N_{c_x}\rightarrow\infty , \\ 
\quad &\mathbb{E}_{x\in{ea}} \bigg( \frac{\textbf{\textit{m}}(x|x\in{c_x})}{\textbf{\textit{m}}(x|x\in{ea})}-1\bigg)^2\rightarrow0, \\
\quad &\mathbb{V}_{x\in{ea}} (\eta (X))\rightarrow0, 
\end{align}
where $\textbf{\textit{m}}(\cdot)$ is the probability density function and $\mathbb{V}$ is the variance operator.
\end{theorem}

Theorem \ref{G-E} give three intuitively reasonable and principled rules to design a good estimator: 
\begin{itemize}
  \item The ensemble generation method should guarantee each $c_x$ has many samples, as suggested by Formula (21). 
  \item The sets $c_x$ and $ea$ should have a high similarity, as suggested by Formula (22). In our method, according to Equation (\ref{bx}) and Equation (\ref{btad}), there are many samples in the intersection of $ea$ and $c_x$, which follows this suggestion.
  \item For an ideal k-HyperEdge Medoids discovering problem, the posterior probabilities in $ea$ should be concentrated, as suggested by Formula (23).
\end{itemize}


%

The proofs of the above theorems are given in the Appendix A.2 to A.4, respectively.

\section{Experimental Analyses}

\begin{table}[!ht]
\begin{center}
\caption{Description of the data sets}
\renewcommand{\arraystretch}{1}
\small
\begin{tabular*}{\hsize}{@{\extracolsep{\fill}}ccccc}
\hline
Number & Data name &  n & d & k\\
\hline
1	&	iris	&	150	&	4	&	3	\\
2	&	wine 	&	178	&	13	&	3	\\
3	&	seeds 	&	210	&	7	&	3	\\
4	&	heart 	&	294	&	12	&	2	\\
5	&	soybean-train 	&	307	&	35	&	18	\\
6	&	ecoli 	&	336	&	7	&	8	\\
7	&	dermatology 	&	366	&	34	&	6	\\
8	&	low-res-spect 	&	531	&	100	&	9	\\
9	&	breast-cancer-wisc-diag 	&	569	&	30	&	2	\\
10	&	energy 	&	768	&	8	&	3	\\
11	&	lbp-riu-gris	&	1022	&	25	&	3	\\
12	&	semeion 	&	1593	&	256	&	10	\\
13	&	statlog-landsat-test 	&	2000	&	36	&	6	\\
14	&	cardiotocography-3clases 	&	2126	&	21	&	3	\\
15	&	statlog-landsat-train 	&	4435	&	36	&	6	\\
16	&	twonorm 	&	7400	&	20	&	2	\\
17	&	mushroom 	&	8124	&	21	&	2	\\
18	&	statlog-shuttle 	&	43500	&	9	&	7	\\
19	&	PenDigits 	&	10992	&	16	&	10	\\
20	&	USPS	&	11000	&	256	&	10	\\
\hline
\end{tabular*}
\label{ucidata}
\end{center}
\end{table}

The CEHM is experimentally analyzed from three aspects: the demonstration of working mechanisms on artificial data, convergence demonstration on real data, and ensemble performance comparison with representative methods. The experimental analyses are conducted on a PCWIN 64 computer with 64G memory.

For a data set, the clustering result set is generated by running the k-means algorithm multiple times with different initial centers and different cluster numbers $k$. For a data set with $n$ samples and $k$ clusters, the numbers of clusters in the base clustering results are randomly selected in the range $\left[k, \left(\max\{\min\{\sqrt{n},50\}, \lceil {\frac{3}{2}}k \rceil\}\right)\right]$. This setting could improve the diversity between the clustering results in the ensemble base.

\subsection{Working Mechanism}

To show the working mechanism of CEHM, four two-dimensional artificial data sets are utilized. The information about these data is shown in Table \ref{artdata} in Appendix A.5. 

Figure \ref{WingNut} and Figure \ref{2d2k} to Figure \ref{Flame} in Appendix A.5. show some of the process results of CEHM on these data sets. In Figure \ref{WingNut} and Figure \ref{2d2k} to Figure \ref{Flame}, the data points that are not in the hypergraph are shown by Gray dot, $\mathcal{E}i$ indicates the result of the k-HyperEdge initialization step, $\mathcal{E}d(1)$ and $\mathcal{E}d(2)$ are two results of the k-HyperEdge diffusion step, $\mathcal{E}a(1)$ and $\mathcal{E}a(2)$ are two results of the k-HyperEdge adjustion step, and $\mathcal{E}_m$ is the discovered k-HyperEdge Medoids. From Figure \ref{WingNut} and Figure \ref{2d2k} to Figure \ref{Flame}, it can be seen that the operation process conforms to the expectations of each stage.


\begin{figure}[t]
\captionsetup[subfigure]{labelformat=empty}
      \centering
       \subfloat[\scriptsize{\textbf{(a)} $\mathcal{E}i$}]{\includegraphics[width=0.15\textwidth]{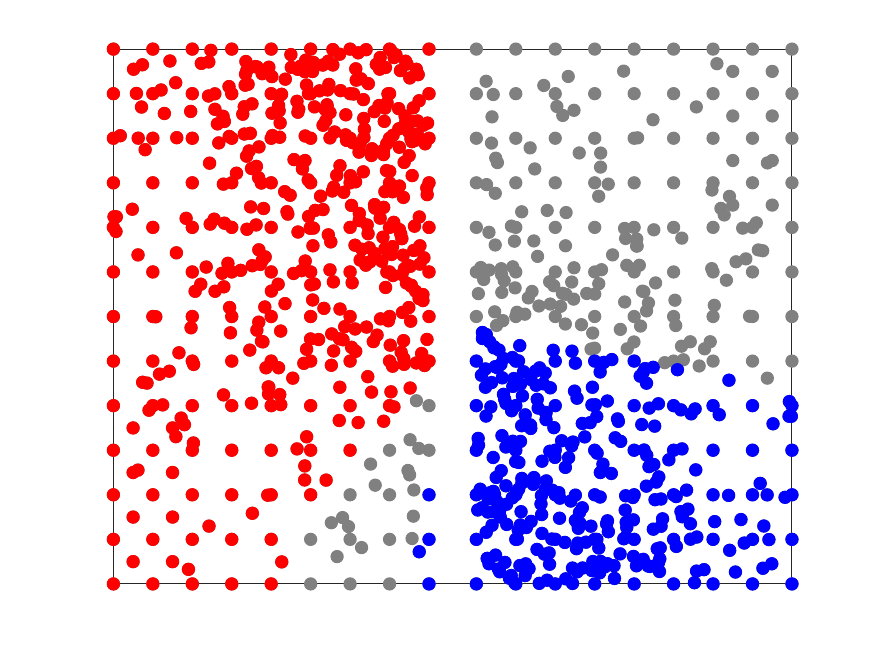}}\
       \subfloat[\scriptsize{\textbf{(b)} $\mathcal{E}d(1)$}]{\includegraphics[width=0.15\textwidth]{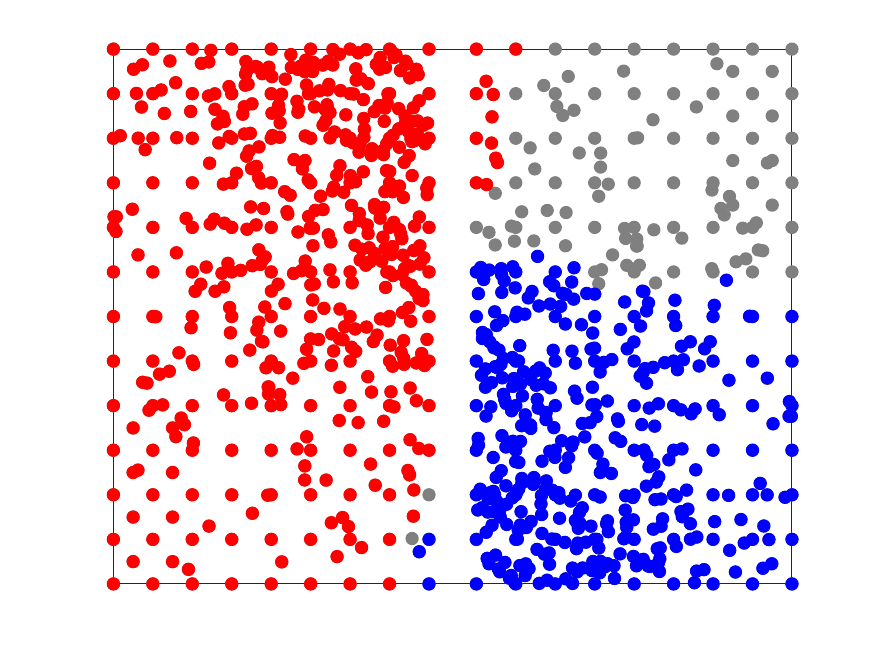}}\
       \subfloat[\scriptsize{\textbf{(c)} $\mathcal{E}d(2)$}]{\includegraphics[width=0.15\textwidth]{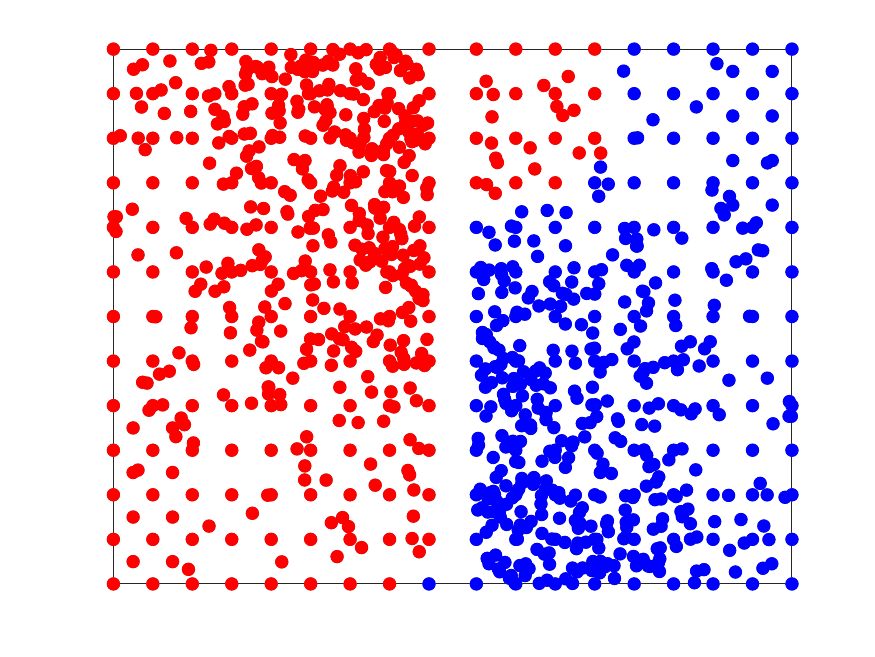}}\\
       \subfloat[\scriptsize{\textbf{(d)} $\mathcal{E}a(1)$}]{\includegraphics[width=0.15\textwidth]{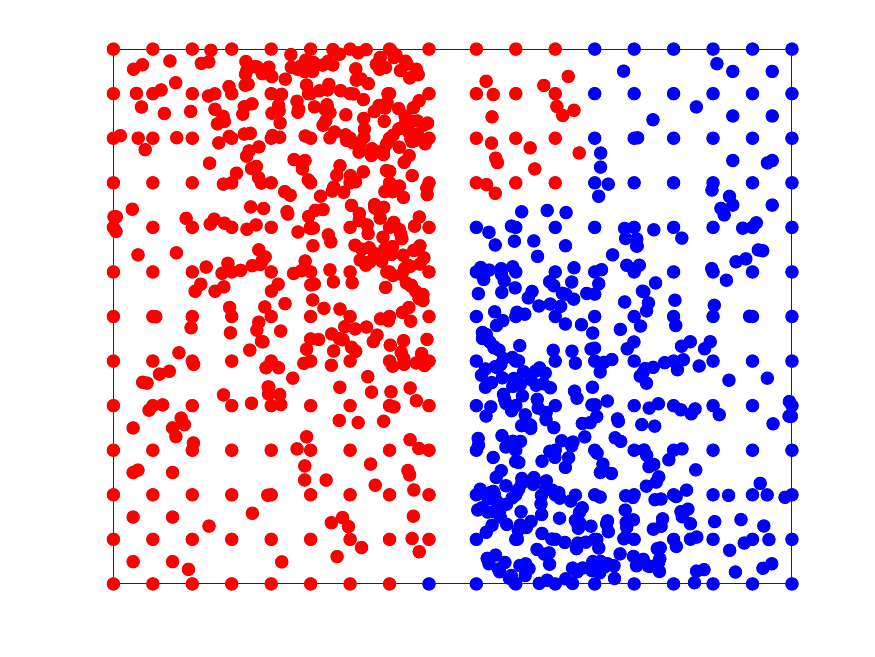}}\
       \subfloat[\scriptsize{\textbf{(e)} $\mathcal{E}a(2)$}]{\includegraphics[width=0.15\textwidth]{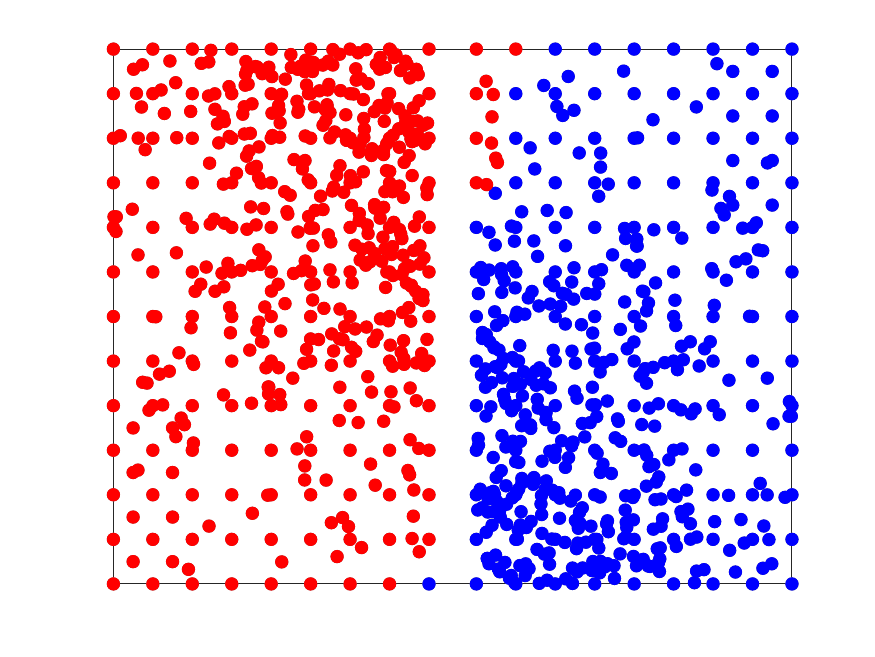}}\
       \subfloat[\scriptsize{\textbf{(f)} $\mathcal{E}_m$}]{\includegraphics[width=0.15\textwidth]{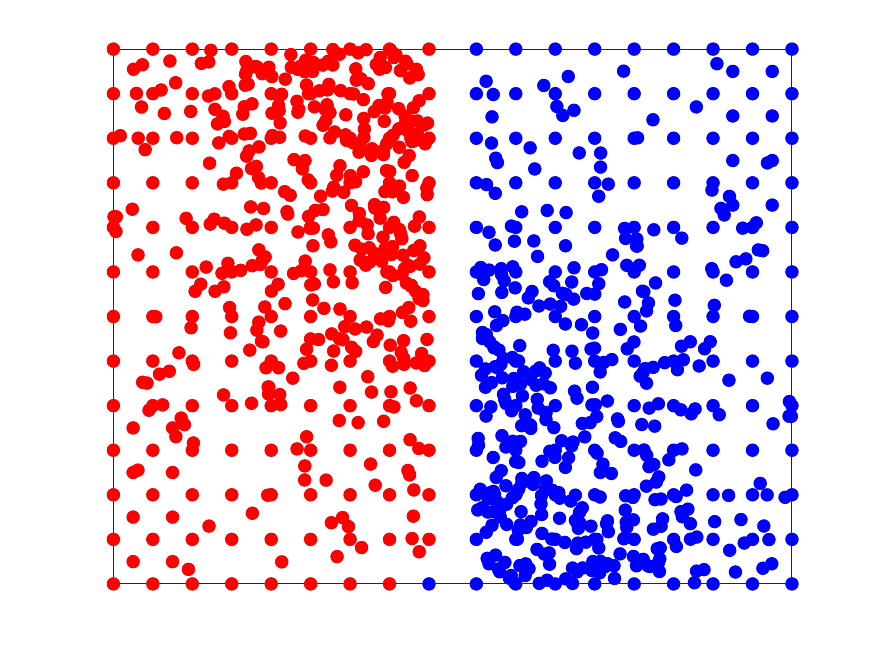}}
\caption{Working mechanism of CEHM on the WingNut data set}
\label{WingNut}
\end{figure}

\subsection{Convergence}

The convergence of CEHM is analyzed on twenty widely used benchmark data sets. The information about these data sets is shown in Table \ref{ucidata}. 

For each data set, we generate $30$ base clustering result sets and then show the loss of the hyperedge that is sequentially generated during the running of CEHM on each base clustering result set. The hyperedge loss is calculated by Equation (\ref{lka}). Figure \ref{conve} shows the results of the convergence analysis on the first 5 data sets. The rest convergence results are shown in Figure \ref{conve_whole} in Appendix A.6. From Figure \ref{conve} and Figure \ref{conve_whole}, it can be seen that the loss of the hyperedge set gradually reduces for each test. The CEHM algorithm can converge quickly. On the test on these data sets, the maximum number of iterations is not greater than $30$.

\begin{figure*}[t]
\captionsetup[subfigure]{labelformat=empty}
      \centering
       \subfloat[\scriptsize{\textbf{(a)} Data 1}]{\includegraphics[width=0.17\textwidth]{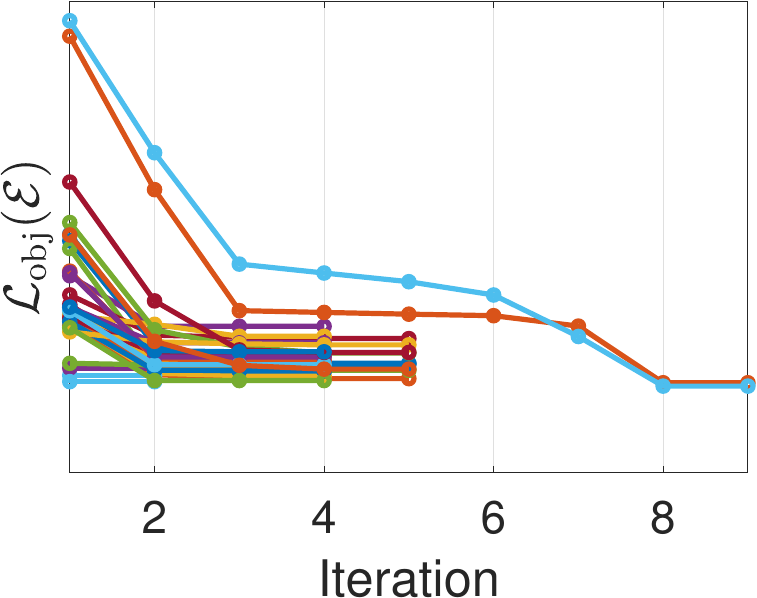}}\
       \subfloat[\scriptsize{\textbf{(b)} Data 2}]{\includegraphics[width=0.17\textwidth]{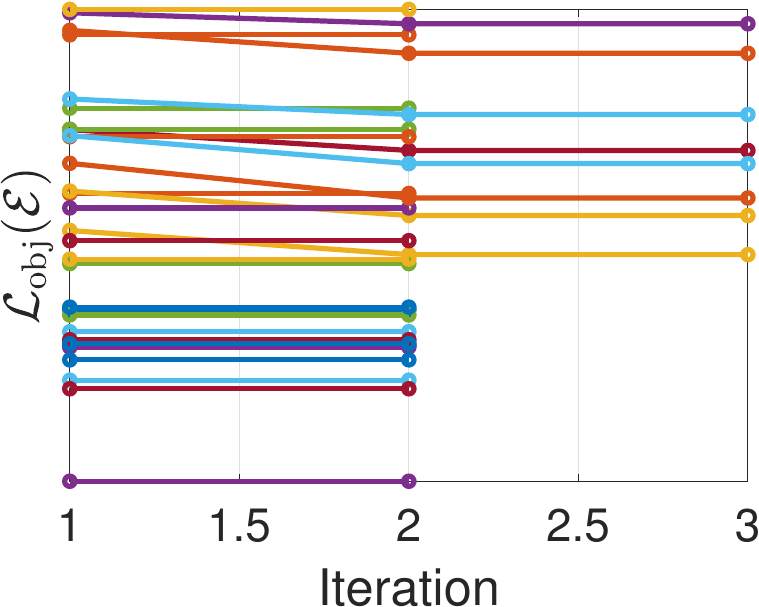}}\
       \subfloat[\scriptsize{\textbf{(c)} Data 3}]{\includegraphics[width=0.17\textwidth]{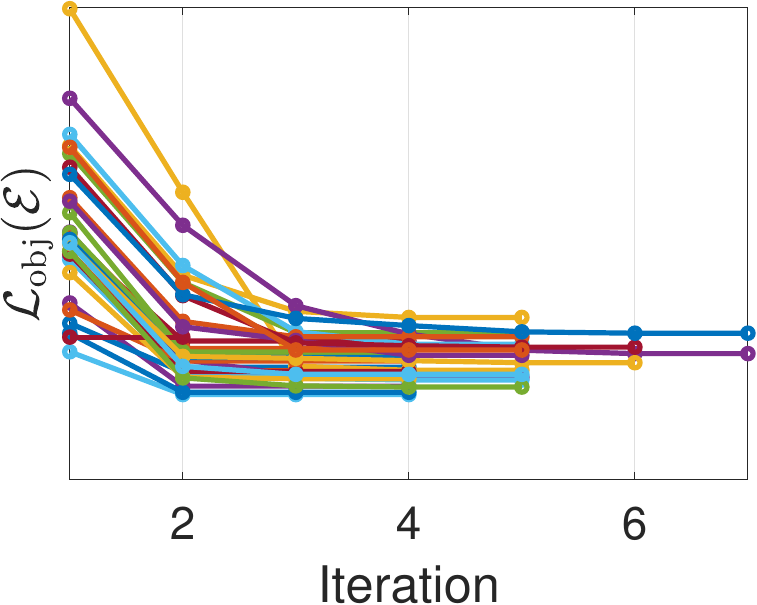}}\
       \subfloat[\scriptsize{\textbf{(d)} Data 4}]{\includegraphics[width=0.17\textwidth]{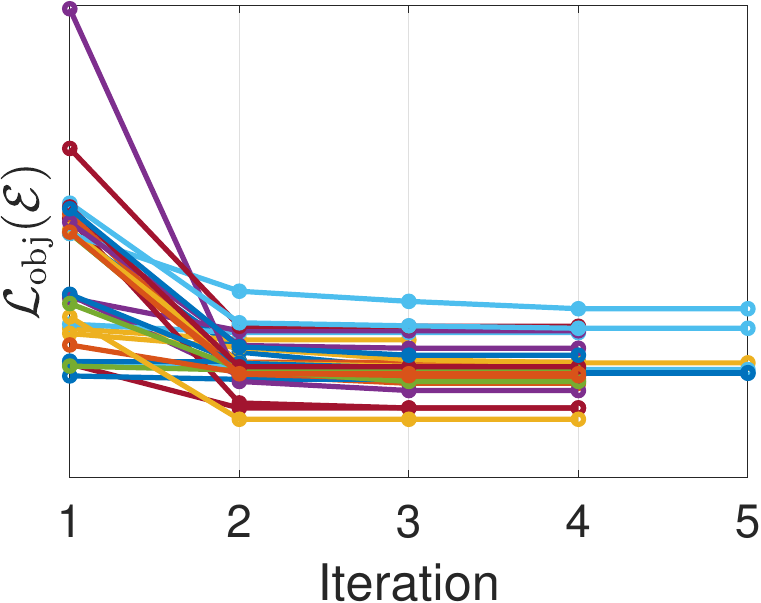}}\
       \subfloat[\scriptsize{\textbf{(e)} Data 5}]{\includegraphics[width=0.17\textwidth]{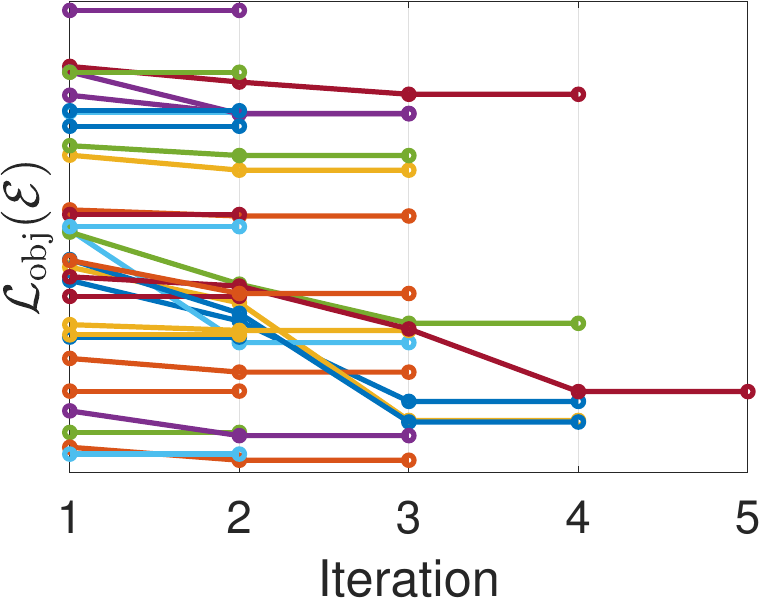}}
\caption{The convergence demonstration of CEHM on the first five data sets}
\label{conve}
\end{figure*}

\begin{figure}[!h] %
\begin{center}
\includegraphics[width=0.43\textwidth, height=0.25\textheight]{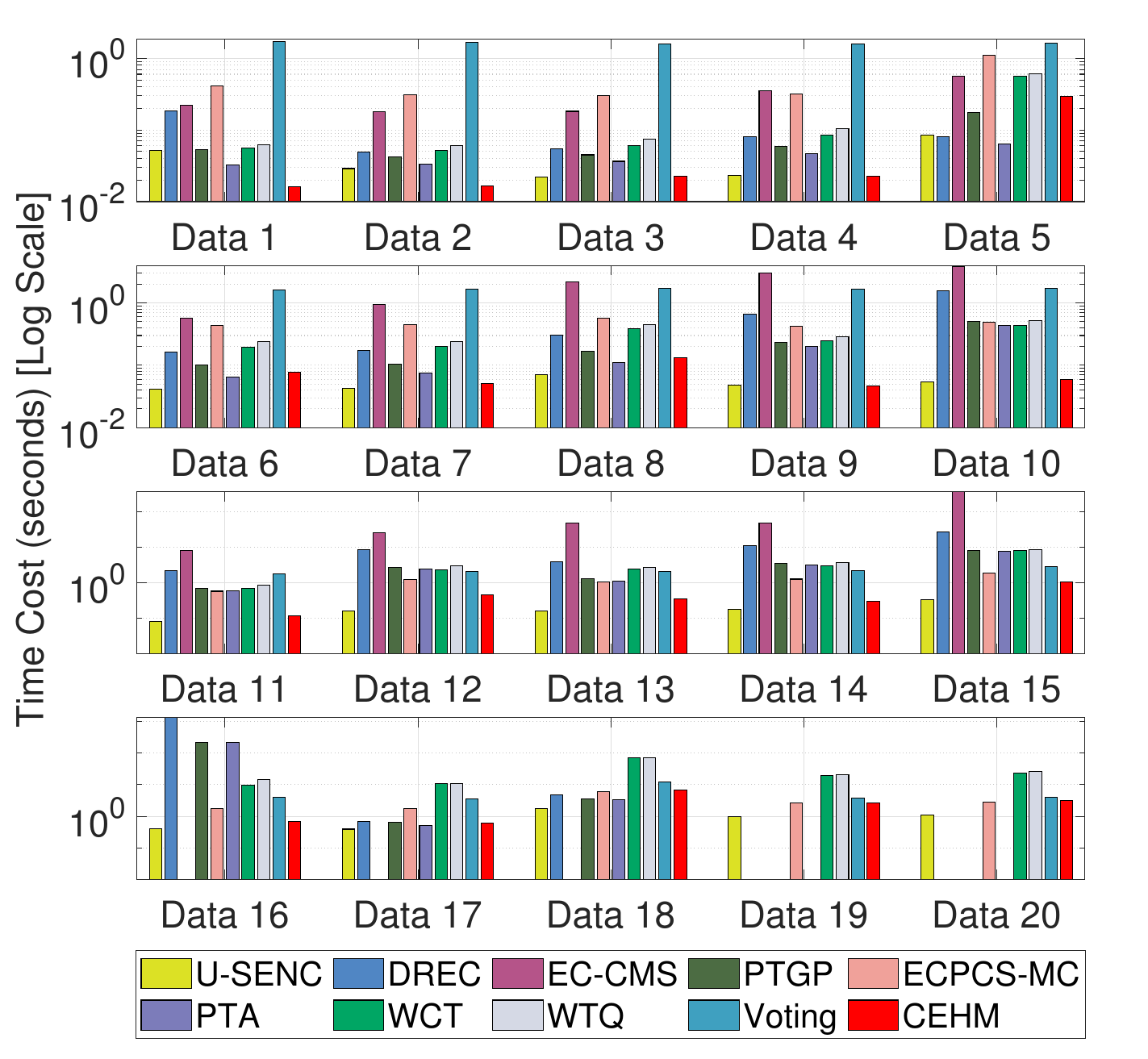}
\caption{The time cost of the compared methods on the twenty data sets}
\label{time}
\end{center}
\end{figure}

\subsection{Ensemble Performance}

To test the clustering ensemble performance of CEHM, nine representative clustering ensemble methods are utilized as reference methods, which are U-SENC \cite{Huang2020}, DREC \cite{Zhou2019Ensemble}, EC-CMS \cite{Jia2023}, PTGP \cite{huang2016robust}, ECPCS-MC \cite{huang2021}, PTA \cite{huang2016robust},  WCT \cite{Iam2011}, WTQ \cite{Iam2011} and Voting \cite{Li2017}. The details about the methods and their settings are in the Appendix A.7.

To evaluate the clustering ensemble result of each method, two of the most widely used external indices are utilized, which are Normalized Mutual Information (NMI) \cite{Strehl} and Adjusted Rand Index (ARI) \cite{hubert1985comparing}. 

\begin{table*}[!ht]
\begin{center}
\caption{The NMI from the compared methods on the twenty data sets}
\renewcommand{\arraystretch}{1}
\small
\begin{tabular*}{\hsize}{@{\extracolsep{\fill}}ccccccccccc}
\hline
Data& U-SENC & DREC & EC-CMS & PTGP & ECPCS-MC & PTA & WCT & WTQ & Voting & CEHM\\
\hline
1&0.6782&0.7123&0.7123&0.6958&0.6907&0.5706&0.7128&0.7080&0.6309&\textbf{\underline{0.7456}}\\
2&0.8961&0.8871&0.8442&0.8833&0.8783&0.8685&0.8947&0.8562&0.8276&\textbf{\underline{0.9039}}\\
3&0.7316&0.6290&0.6770&0.6426&0.6960&0.6022&0.6915&0.6788&0.7297&\textbf{\underline{0.7320}}\\
4&0.3054&0.2644&0.3033&0.1422&0.2835&0.1179&0.2790&0.2106&0.1797&\textbf{\underline{0.3072}}\\
5&0.7373&0.7277&0.7262&0.6916&0.7366&0.6957&0.7135&0.7160&0.6952&\textbf{\underline{0.7420}}\\
6&0.6178&0.6048&0.6482&0.5757&0.6493&0.6063&0.6000&0.5924&0.5371&\textbf{\underline{0.7007}}\\
7&0.8910&0.8757&0.8572&0.8838&0.8952&0.8873&0.8532&0.8336&0.8109&\textbf{\underline{0.9109}}\\
8&0.4278&0.4228&0.4633&0.4269&0.4619&0.4346&0.4045&0.4069&0.3967&\textbf{\underline{0.4674}}\\
9&0.6074&0.6315&0.2299&0.5885&0.5570&0.5398&0.5710&0.5841&0.4522&\textbf{\underline{0.6511}}\\
10&\textbf{\underline{0.5853}}&\textbf{\underline{0.5853}}&\textbf{\underline{0.5853}}&\textbf{\underline{0.5853}}&0.5853&\textbf{\underline{0.5853}}&0.5580&0.5612&0.3762&\textbf{\underline{0.5853}}\\
11&0.3850&0.4118&0.3156&0.3210&0.3234&0.3768&0.3605&0.3539&0.3400&\textbf{\underline{0.4718}}\\
12&0.6424&0.6573&0.6369&0.6463&0.6457&0.6608&0.6265&0.5987&0.6008&\textbf{\underline{0.6665}}\\
13&0.6175&0.6142&0.5348&0.5998&0.6190&0.5864&0.6119&0.5748&0.5621&\textbf{\underline{0.6316}}\\
14&0.2035&0.1365&0.1352&0.1044&0.1865&0.1137&0.1166&0.1694&0.1088&\textbf{\underline{0.2391}}\\
15&0.6299&0.6479&0.6071&0.6376&0.6510&0.6630&0.6134&0.5642&0.5322&\textbf{\underline{0.6687}}\\
16&0.8400&0.8384&OM&0.8378&0.8426&0.8204&0.8411&0.8400&0.8432&\textbf{\underline{0.8444}}\\
17&0.1907&0.2196&OM&0.2214&0.1902&0.0925&0.1909&0.2486&0.2272&\textbf{\underline{0.3359}}\\
18&0.4931&0.3352&OM&0.4019&0.1002&0.4253&0.4024&0.3353&0.3043&\textbf{\underline{0.5994}}\\
19&0.7526&OM&OM&OM&0.7550&OM&0.7273&0.6911&0.6752&\textbf{\underline{0.7567}}\\
20&0.5434&OM&OM&OM&0.5603&OM&0.5613&0.4983&0.5025&\textbf{\underline{0.5818}}\\
\hline
\end{tabular*}
\label{index_nmi}
\end{center}
\end{table*}

The ensemble performance results are shown in Table \ref{index_nmi} and Table \ref{index_ari} (in Appendix A.8). The tables respectively show the average NMI and AIR of the $30$ times for each method on each data set. The maximum index value of each comparison is marked in bold type with underlining. The OM means the corresponding method out of memory during handling the data. From the tables, it is easy to see that the CEHM is marked on most of the data sets, which show the effectiveness of CEHM in handling the clustering ensemble problem on these data sets. . 


We statistically analyze the experimental results that are shown in the Table \ref{index_nmi} and Table \ref{index_ari} (in Appendix A.8). The analysis methods and results are included in Appendix A.9.

\subsection{Time Cost}

The time complexity of CEHM consists of three parts, which correspond to the three steps. The time complexity of the k-HyperEdge initialization step is $O({n_c}^2)$ \cite{Tiwari2020}. The time complexity of the k-HyperEdge diffusion step and the k-HyperEdge adjustion are both $O(tnn_c)$. Then the total time the k-HyperEdge adjustion is $O({n_c}^2+2tnn_c)$. The time costs of these methods on the twenty data sets are also compared. Figure \ref{time} shows the average running time of each method on the twenty data sets. From Figure \ref{time}, it can be seen that CEHM consumes the least running time except for U-SENC, a method specifically designed for large-scale data with the downsampling techniques. The results illustrate the efficiency of CEHM.
%

\section{Conclusion}

In this paper, we proposed k-HyperEdge Medoids for clustering ensemble. The k-HyperEdge Medoids were defined as nonoverlapping hyperedges, and each hyperedge in the hypergraph could find a similar hyperedge medoid. We solved the clustering ensemble problem by discovering a set of k-HyperEdge Medoids. The proposed discovery method contains k-HyperEdge initialization, k-HyperEdge diffusion, and k-HyperEdge adjustion. Concretely, a set of initial sub-hyperedges is selected in the k-HyperEdge initialization step, and the quality of the chosen hyperedges is improved by diffusion and adjustment guided by a hyperedge loss function that was reduced by assigning samples to the hyperedge with the highest degree of belonging. The rationality of the solution, the assignment method, and the belonging degree estimation method are illustrated by theoretical analysis. The working mechanism, convergence, effectiveness, and efficiency were illustrated by experimental analysis. In the future, it is interesting to develop novel learning methods to construct the k-HyperEdge Medoids.

\bibliography{CECH.bib}

\clearpage

\appendix

\section{Appendix}

The Appendix includes:

\begin{itemize}
  \item A.1 The algorithms of k-HyperEdge initialization, k-HyperEdge diffusion, k-HyperEdge adjustion, and CEHM
  \item A.2 The proofs of Theorem \ref{the_bound}
  \item A.3 The proofs of Theorem \ref{the_re}
  \item A.4 The proofs of Theorem \ref{G-E}
  \item A.5 The working mechanism of CEHM on the artificial data sets
  \item A.6 The convergence results on the last fifteen data sets
  \item A.7 The compared methods and their settings
  \item A.8 The ARI from the compared methods on the twenty data sets
  \item A.9 The Statistical analysis about the experimental results
\end{itemize}

\subsection{A.1 The algorithms of k-HyperEdge initialization, k-HyperEdge diffusion, k-HyperEdge adjustion, and CEHM} \label{ce:re}

The processes of the k-HyperEdge initialization step are shown as Algorithm \ref{KHI}.

\begin{algorithm}[!ht]
\caption{k-HyperEdge initialization}
\label{KHI}
 \begin{algorithmic}
\STATE \textbf{INPUT:} $\Pi=\{\pi^1, \pi^2,\ldots, \pi^l\}$ and $k$
\STATE \textbf{OUTPUT:} $\mathcal{E}i=\{ei_1, ei_2, \ldots ei_k\}$
\STATE \textbf{Process:}
 \end{algorithmic}
\begin{algorithmic}[1]
\STATE Representing $\Pi$ as $\mathcal{E}(\Pi_c)=\{c_1, c_2, \ldots, c_{n_c}\}$
\STATE Running k-medoids algorithm on $\mathcal{E}(\Pi_c)$, obtaining $\mathcal{E}i=\{ei_1, ei_2, \ldots ei_k\}$
\FOR {$ei_i \in \mathcal{E}i$, $ei_j\in\mathcal{E}i$, $i\neq j$}
\STATE $ei_i = ei_i \setminus (ei_i\cap ei_j)$
\STATE $ei_j = ei_j \setminus (ei_i\cap ei_j)$
\ENDFOR
 \end{algorithmic}
\end{algorithm}

The processes of the k-HyperEdge diffusion step are shown as Algorithm \ref{KHD}.

\begin{algorithm}[!h]
\caption{k-HyperEdge diffusion}
\label{KHD}
 \begin{algorithmic}
\STATE \textbf{INPUT:} $\mathcal{E}(\Pi_c)=\{c_1, c_2, \ldots, c_{n_c}\}$, and $\mathcal{E}i=\{ei_1$, $ei_2$, $\ldots$, $ei_k\}$.
\STATE \textbf{OUTPUT:} $\mathcal{E}d=\{ed_1, ed_2, \ldots ed_k\}$.
\STATE \textbf{Process:}
 \end{algorithmic}
\begin{algorithmic}[1]
\STATE Initializing $\mathcal{E}d = \mathcal{E}i$
\STATE Obtaining $n_a$ based on Formula (\ref{n_a})
\WHILE {$n_a\neq n$}
\FOR {$i=1$ to $n$}
\STATE Calculating $ea^*(x_i)$ based on Formula (\ref{hx}).
\STATE Calculating $m(x_i)$ based on Formula (\ref{mx}).
\ENDFOR
\STATE Obtaining $n_s$ based on Formula (\ref{na})
\FOR {$i=1$ to $k$}
\STATE Updating $ed_i$ based on Formula (\ref{hdif}).
\ENDFOR
\STATE Obtaining $n_a$ based on Formula (\ref{n_a})
\ENDWHILE
 \end{algorithmic}
\end{algorithm}

The processes of the k-HyperEdge adjustion step are shown as Algorithm \ref{KHA}.

\begin{algorithm}[!h]
\caption{k-HyperEdge adjustion}
\label{KHA}
 \begin{algorithmic}
\STATE \textbf{INPUT:} $\mathcal{E}(\Pi_c)=\{c_1, c_2, \ldots, c_{n_c}\}$, $\mathcal{E}d=\{ed_1$, $ed_2$, $\ldots$, $ed_k\}$.
\STATE \textbf{OUTPUT:} $\mathcal{E}a=\{ea_1, ea_2, \ldots ea_k\}$.
\STATE \textbf{Process:}
 \end{algorithmic}
\begin{algorithmic}[1]
\STATE Initializing $t=1$, $\mathcal{E}a^t = \mathcal{E}d$, and $\mathcal{L}_{\text{adj}}(\mathcal{E}a^0)=n$
\STATE Calculating  $\mathcal{L}_{\text{adj}}(\mathcal{E}a^t)$ based on Formula (\ref{lka})
\WHILE {$\mathcal{L}_{\text{adj}}(\mathcal{E}a^t) \neq \mathcal{L}_{\text{adj}}(\mathcal{E}a^{t-1})$}
\STATE Setting $t=t+1$
\FOR {$i=1$ to $n$}
\STATE Calculating $ea^{t*}(x_i)$ based on Formula (\ref{btad}).
\ENDFOR
\FOR {$i=1$ to $k$}
\STATE Updating ${ea^t}_i$ based on Formula (\ref{upedge}).
\ENDFOR
\STATE Calculating $\mathcal{L}_{\text{adj}}(\mathcal{E}a^t)$ based on Formula (\ref{lka})
\ENDWHILE
 \end{algorithmic}
\end{algorithm}

The processes of CEHM are shown as Algorithm \ref{CEHM}.

\begin{algorithm}[!h]
\caption{CEHM}
\label{CEHM}
 \begin{algorithmic}
\STATE \textbf{INPUT:} $\Pi=\{\pi^1, \pi^2,\ldots, \pi^l\}$ and $k$
\STATE \textbf{OUTPUT:} $\pi^*=\{c_1, c_2, \ldots, c_k\}$
\STATE \textbf{Process:}
\end{algorithmic}
\begin{algorithmic}[1]
\STATE $\mathcal{E}i \leftarrow$ Algorithm 1 $(\Pi, k)$
\STATE $\mathcal{E}d \leftarrow$ Algorithm 2 $(\Pi, \mathcal{E}i)$
\STATE $\mathcal{E}a \leftarrow$ Algorithm 3 $(\Pi, \mathcal{E}d)$
\STATE $\mathcal{E}m = \mathcal{E}a$
\STATE Obtaining $\pi^*$ based on Formula (\ref{piedge})
 \end{algorithmic}
\end{algorithm}

\subsection{A.2 The proofs of Theorem \ref{the_bound}} \label{sec:pr1}

Before proving Theorem \ref{the_bound}, we first introduce the following lemma.

\begin{lemma} \label{le_in}
Assuming a base cluster set $\Pi_c$ and three heyperedges $ea_i$, $ea_j$, and $ea_r$. Let \begin{align}
dist(ea_i, ea_j) = n- |ea_i \cap ea_j|,
\end{align}


\begin{align}
\mathcal{\phi}(ea_i; ea_j) = \sum_{x\in ea_j} \left(n- \frac{1}{l}\sum_{x\in c_x} |c_x \cap ea_i| \right).
\end{align}

We can get the following three inequalities
\begin{align}\label{triangle1}
dist (ea_i,ea_j)+dist (ea_j,ea_r)&\ge dist (ea_i,ea_r);
\end{align}

\begin{align}\label{triangle2}
dist (ea_i,ea_j)-dist (ea_j,ea_r)&\le dist(ea_i,ea_r);
\end{align}

\begin{align}\label{triangle3}
\mathcal{\phi}(ea_i; ea_j) \ge |ea_j| dist (ea_i, ea_j)-\mathcal{\phi}(ea_j; ea_j).
\end{align}
\end{lemma}

\begin{proof}

For the first inequality, we have

\begin{align}\label{triangle_1}
&dist(ea_i, ea_j)+dist(ea_j, ea_r) - dist (ea_i,ea_r) \\ \notag
= & n + |ea_i \cap ea_r| - \left(|ea_i \cap ea_j| + |ea_j \cap ea_r|\right) \\ \notag
= & n + |ea_i \cap ea_r|  \\ 
& - \left(|ea_j \cap (ea_i \cup ea_r)| + |ea_j \cap ea_i \cap ea_r| \right) \\ \notag
\geq & n + |ea_i \cap ea_r| -  \left( |ea_i \cup ea_r| + |ea_i \cap ea_r| \right)\\ \notag
= & n - |ea_i \cup ea_r|\\ \notag
\geq & 0
\end{align}

Then, the first inequality holds.

The second inequality can be express as

\begin{align}\label{triangle_2}
& dist (ea_i,ea_j)-dist (ea_j,ea_r) \le dist(ea_i,ea_r) \\ \notag
\Leftrightarrow & dist(ea_i,ea_r) + dist (ea_j,ea_r)\geq dist (ea_i,ea_j)
\end{align}

According to Inequality (\ref{triangle1}), Inequality (\ref{triangle2}) holds true.

For Inequality (\ref{triangle3}), we have

\begin{align}\label{t3_1}
& \mathcal{\phi}(ea_i; ea_j)  \\ \notag
= & \sum_{x\in ea_j} \left( n- \frac{1}{l}\sum_{x\in c_x} |c_x \cap ea_i| \right) \\ \notag
= & \frac{1}{l} \sum_{x\in ea_j} \sum_{x\in c_x} \left( n- |c_x \cap ea_i|\right)\\ \notag
= & \frac{1}{l} \sum_{x\in ea_j} \sum_{x\in c_x} dist(c_x, ea_i).
\end{align}

In a similar way

\begin{align}\label{t3_2}
\mathcal{\phi}(ea_j; ea_j) = \frac{1}{l} \sum_{x\in ea_j} \sum_{x\in c_x} dist(c_x, ea_j).
\end{align}

In addition

\begin{align}\label{t3_3}
 & |ea_j| dist (ea_i, ea_j) \\ \notag
= & \sum_{x\in ea_j} dist (ea_i, ea_j) \\ \notag
= & \frac{1}{l} \sum_{x\in ea_j} \sum_{x\in c_x} dist (ea_i, ea_j).
\end{align}

Combining Equation (\ref{t3_1}), (\ref{t3_2}) and (\ref{t3_3}), the Inequality (\ref{triangle3}) is transformed to proof

\begin{align} \label{t3_4}
dist(c_x, ea_i) \geq dist (ea_i, ea_j)- dist(c_x, ea_j).
\end{align}

According to Inequality (\ref{triangle2}), Inequality (\ref{t3_4}), then Inequality (\ref{triangle3}) holds.
\end{proof}

Based on Lemma \ref{le_in}, we then prove the Theorem \ref{the_bound}.

\noindent\textbf{Theorem \ref{the_bound}.}
\emph{
let $\mathcal{E}a^0=\{{ea^0}_1, {ea^0}_2,\ldots, {ea^0}_k\}$ be a set of hyperedges produced by a g-approximate algorithm based on a base cluster set $\Pi_c$, and Let 
\begin{align}
dist(ea_i, ea_j) = n- |ea_i \cap ea_j|,
\end{align}
\begin{align}
\mathcal{\phi}(ea_i; ea_j) & = \sum_{x\in ea_j} \left(n- b(x, ea_i) \right) \\ \notag
& = \sum_{x\in ea_j} \left(n- \frac{1}{l}\sum_{x\in c_x} |c_x \cap ea_i| \right),
\end{align}
then for any hyperedge set, denoted by $\mathcal{E}a=\{{ea}_1, {ea}_2,\ldots, {ea}_k\}$,  we have $\forall ea_i$, $\exists {ed^0}_i$ s.t. $dist(ea_i, {ea^0}_i) \le \frac{(g+1) \phi (ea_i; *)}{|ea_i|}$.
}

\begin{proof}

Suppose $\forall ea_i$, $\exists {ed^0}_i$ s.t. $dist(ea_i, {ea^0}_i) > \frac{(g+1) \phi (ea_i)}{|ea_i|}$.

Based on Inequality (\ref{triangle3}), we draw into the following inequality and deduction

\begin{align} \label{theorem_r1}
\mathcal{\phi}({ea^0}_i; ea_i) & \ge |ea_i| dist ({ea^0}_i, ea_i)-\mathcal{\phi}(ea_i; ea_i) \\ \notag
&\ge|ea_i| \frac{(g+1) \phi (ea_i; *)}{|ea_i|}-\mathcal{\phi}(ea_i; ea_i) \\ \notag
&\ge(g+1)\phi (ea_i; *)-\phi (ea_i; *) \\ \notag
&\ge g \phi (ea_i; *),
\end{align}
where $\phi (ea_i; *)$ denotes the optimal cost. The Inequality (\ref{theorem_r1}) contradicting the g-approximate solution.
\end{proof}


\subsection{A.3 The proofs of Theorem \ref{the_re}} \label{sec:pr2}

\noindent\textbf{Theorem \ref{the_re}.}
\emph{
Assuming the optimal hyperedge of sample $x$ is the $i$-th hyperedge: $ea^*(x)=ea_i$; assuming a 
hyperedge sets $\mathcal{E}a^1=\{{ea^1}_1, {ea^1}_2,\ldots, {ea^1}_k\}$ and $x$ is in the $i$-th hyperedge in $\mathcal{E}a^1$: $x\in {ea^1}_i$; assuming another hyperedge set $\mathcal{E}a^2$ based on $\mathcal{E}a^1$, and in $\mathcal{E}a^2$, $x$ is moved from the $i$-th hyperedge into the $j$-th hyperedge: for $\forall p \in \{1,2,..,k\}/\{i,j\}$, ${ea^2}_i={ea^1}_i/x$, ${ea^2}_j={ea^1}_j\cup x$. Then, we have $\mathcal{L}_{\text{adj}}(\mathcal{E}a^1)\leq \mathcal{L}_{\text{adj}}(\mathcal{E}a^2)$.
}

\begin{proof}

We first categorize and discuss the values of $b\left(y, {ea^2}(y)\right)$.

When $y\notin {ea^2}_i$, and $y\notin {ea^2}_j$, then based on the assumption $\forall p\neq i \neq j$, ${ea^2}_p={ea^1}_p$, and according to Equation (\ref{bx}), it is easy to obtain

\begin{align}
b\left(y, {ea^2}(y)\right)=b\left(y, {ea^1}(y)\right).
\end{align}

When $y\in {ea^2}_i$, we have

\begin{align}
  & b\left(y, {ea^2}_i\right) \notag\\
= & \frac{1}{l}\sum_{y\in c_y} |c_y \cap {ea^2}_i| \notag\\
= & \frac{1}{l}\sum_{y\in c_y} |c_y \cap ({ea^1}_i/x)| \notag\\
= & \frac{1}{l}\sum_{y\in c_y} \left(|c_y \cap {ea^1}_i|- |c_y \cap x|\right) \notag\\
= & b\left(y, {ea^1}_i\right) - \frac{1}{l}\sum_{y\in c_y} |c_y \cap x| \notag\\
= & b\left(y, {ea^1}_i\right) - \frac{1}{l}\sum_{x, y\in c_x} 1.
\end{align} 

When $y\in {ea^2}_j$, we have

\begin{align}
  & b\left(y, {ea^2}_j\right) \notag\\
= & \frac{1}{l}\sum_{y\in c_y} |c_y \cap ({ea^1}_j\cup x)| \notag\\
= & b\left(y, {ea^1}_j\right) + \frac{1}{l}\sum_{y\in c_y} |c_y \cap x| \notag\\
= & b\left(y, {ea^1}_i\right) + \frac{1}{l}\sum_{x, y\in c_x} 1.
\end{align} 

Then, based on Equation (\ref{lka}), 

\begin{align} \label{adj12}
  & \mathcal{L}_{\text{adj}}(\mathcal{E}a^2) \\ \notag
= & \sum_{y \in \mathcal{E}a^2} \left(1- b(y, ea^2(y))\right) \\ \notag
= & \mathcal{L}_{\text{adj}}(\mathcal{E}a^1) + \frac{1}{l}\sum_{y\in {ea^2}_i, x, y\in c_y} 1 - \frac{1}{l}\sum_{y\in {ea^2}_j, x, y\in c_y} 1 
\end{align}

Based on the assumption $ea^*(x)=ea_i$ and Equation (\ref{bx}), we have

\begin{align}\label{bx12}
  & b(x, {ea^1}_i) \geq b(x, {ea^1}_j) \\ \notag
\Rightarrow & \sum_{x \in c_x} |c_x \cap {ea^1}_i| \geq \sum_{x \in c_x} |c_x \cap {ea^1}_j| \\ \notag
\Rightarrow & \sum_{y\in {ea^1}_i, x, y\in c_x} 1 \geq \sum_{y\in {ea^2}_j, x, y\in c_x} 1
\end{align}

Then, combining Formula (\ref{adj12}) and Formula (\ref{bx12}), we have $\mathcal{L}_{\text{adj}}(\mathcal{E}a^2) \geq \mathcal{L}_{\text{adj}}(\mathcal{E}a^1)$.
\end{proof}

\subsection{A.4 The proofs of Theorem \ref{G-E}}

\noindent\textbf{Theorem \ref{G-E}.}
\emph{
Consider an estimator $\hat{\eta}_N(X)$ of $\eta(X)$ as defined above. 
Suppose that the second moment of $\eta(X)$ of $ea$ exists: 
$\mathbb{E}_{x\in ea}  \eta(X) ^2 <+\infty $ and the difference
in probability density between set $ea$ and the whole set $\mathcal{X}$ is bounded:
that is, there exists $C$, such that 
$\mathbb{E}_{x\in ea}
\left(\frac{\textbf{\textit{m}}(x|x\in\mathcal{X})}{\textbf{\textit{m}}(x|x\in ea)}\right)^2<C$.
Then $\mathbb{E}|\hat{\eta}_N(X)-\eta(X)|\rightarrow 0$ if
\begin{align}
&N_{c_x}\rightarrow\infty , \notag\\
&\mathbb{E}_{x\in{ea}} \bigg( \frac{\textbf{\textit{m}}(x|x\in{c_x})}{\textbf{\textit{m}}(x|x\in{ea})}-1\bigg)^2\rightarrow0, \notag\\
&\mathbb{V}_{x\in{ea}} (\eta (X))\rightarrow0,
\end{align}
where $\textbf{\textit{m}}(\cdot)$ is the probability density function and $\mathbb{V}$ is the variance operator.
}

We prove Theorem \ref{G-E} by providing an upper bound on the approximation error in estimating probability through nearest neighbors, i.e. Theorem \ref{G-E1}. The Theorem \ref{G-E} is an application of Theorem \ref{G-E1} in a specific scenario when the $A$ in Theorem \ref{G-E1} is $c_x$ and the $B$ is $ea$.

\subsubsection{A.4.1 Approximation error in estimating probability through nearest neighbors}

Many methods estimate the properties of a sample based on the properties of its neighbors \cite{wang2020learning}. Suppose there is a cell $\mathcal{B} \in \mathcal{R}^d$ equipped with the properties of $B$, what we are interested in is the probability that a given sample $X=x$ has property $B$, i.e.
\begin{equation}
 \eta (x)= \mathbb{P}\{x\in \mathcal{B}|X=x\}.
\end{equation}

However, because the underlying distribution of the samples is generally unknown, it is not feasible to estimate the probability that a single sample point with a certain property. As a compromise, one can count ﻿the number of samples which have the property $B$ in another cell $\mathcal{A} \in \mathcal{R}^d$ to approximate $\eta (x)$. Suppose there are $N$ samples $\{X_1,...,X_N\}$, the approximation $\hat{\eta}_N(x)$ based on $\mathcal{A}$ is
\begin{equation}
\hat{\eta}_N(x)= \frac{1}{N_\mathcal{A}}\sum_{i:X_i\in\mathcal{A}} \mathbb{I}\{X_i\in \mathcal{B}\},
\end{equation}
where $N_\mathcal{A}$ denote the number of samples falling in $\mathcal{A}$, i.e.
\begin{equation}
N_\mathcal{A} =\sum_{i=1}^N \mathbb{I}\{X_i\in \mathcal{A}\},
\end{equation}
and $\mathbb{I}$ is the indicator function.

Intuitively, the efficiency of the above estimator depends on the number of available samples and difficulty degree of estimating property $B$.
Besides, to make the method of estimating $B$ via $A$ feasible, there may require a certain similarity between property $A$ and property $B$.
In this section, we theoretically provide some criteria to judge the quality of the estimator.
These criteria provide us novelty insights into evaluating or designing estimators.

To evaluate of the estimator $\hat{\eta}_N(x)$, we explore the upper bound of the estimation error $\mathbb{E}|\hat{\eta}_N(X)-\eta(X)|$, which is reflected by the following Theorem \ref{G-E1}. From Theorem \ref{G-E1}, we can obtain that if the number of samples in $\mathcal{A}$ is large enough, the probability density distributions of sets $\mathcal{A}$ and $\mathcal{B}$ are close and the variance of posterior probability over $\mathcal{B}$ is small, then the $\hat{\eta}_N(x)$ is a good estimator. This conclusion is consistent with the intuitive idea.


\begin{theorem}
\label{G-E1} 
Consider an estimator $\hat{\eta}_N(X)$ of $\eta(X)$ as defined above. 
Suppose that the second moment of $\eta(X)$ of $\mathcal{B}$ exists: 
$\mathbb{E}_{x\in\mathcal{B}}  \eta(X) ^2 <+\infty $ and the difference
in probability density between set $\mathcal{B}$ and the whole set $\mathcal{X}$ is bounded:
that is, there exists $C$, such that 
$\mathbb{E}_{x\in\mathcal{B}}
\left(\frac{\textbf{\textit{m}}(x|x\in\mathcal{X})}{\textbf{\textit{m}}(x|x\in\mathcal{B})}\right)^2<C$.
Then $\mathbb{E}|\hat{\eta}_N(X)-\eta(X)|\rightarrow 0$ if
\begin{align}
&N_\mathcal{A}\rightarrow\infty , \notag\\
&\mathbb{E}_{x\in\mathcal{B}} \bigg( \frac{\textbf{\textit{m}}(x|x\in\mathcal{A})}{\textbf{\textit{m}}(x|x\in\mathcal{B})}-1\bigg)^2\rightarrow0, \notag\\
&\mathbb{V}_{x\in\mathcal{B}} (\eta (X))\rightarrow0,
\end{align}
where $\textbf{\textit{m}}(\cdot)$ is the probability density function and $\mathbb{V}$ is the variance operator.
\end{theorem}

\begin{proof}
To gain insight into the estimation error, we firstly introduce two terms.
Let
 \begin{align}
\bar{\eta}_B(X)&=\mathbb{E}\{\eta (X)|X\in\mathcal{B}\}, \notag \\
\bar{\eta} (X)&= \mathbb{E}\{\eta(X)|X\in\mathcal{A}\}, 
\end{align}
s.t. $N_\mathcal{A}\bar{\eta} (X)$ is an integer. 

By the triangle inequality, we know the error can be bounded by:
 \begin{align}
\label{errorbound_A}
&\mathbb{E}|\hat{\eta}_N(X)-\eta(X)| \notag\\
\leq &\mathbb{E}|\hat{\eta}_N(X)-\bar{\eta}(X)|+ \mathbb{E}|\bar{\eta}(X)- \eta (X)|\notag \\
\leq & \underbrace{\mathbb{E}|\hat{\eta}_N(X)-\bar{\eta}(X)|}_{\text{Term 1}} +\underbrace{\mathbb{E}|\bar{\eta}(X)-\bar{\eta}_B(X)|}_{\text{Term 2}} \notag \\
& +\underbrace{\mathbb{E}|\bar{\eta}_B(X)- \eta (X)|}_{\text{Term 3}}.
\end{align}
Obviously, the terms in the above bound are the estimation error of samples, the approximation error the set $\mathcal{A}$ and
the absolute deviation the probability and its expectation in set $\mathcal{B}$ in sequence. Next, we provide upper bounds on the these three terms.

{\bf{Term 1. Bound on the estimation error of the samples.}}

By simply calculation, we find that $\mathbb{E} \hat{\eta}_N(X) =  \bar{\eta} (X)$.
That is, $N_\mathcal{A}\hat{\eta}_N(X)$ follows a binomial distribution with parameters and $N_\mathcal{A}$ and $\bar{\eta} (X)$.
According to Theorem \ref{binomial} in Appendix A.4.2, we have:
\begin{equation}
\mathbb{E}|\hat{\eta}_N(X)-\bar{\eta} (X)|\leq \sqrt{\frac{ 1}{2\pi   N_\mathcal{A}}} .
\end{equation}
When $ N_\mathcal{A}\rightarrow\infty$, we obtain $\mathbb{E}|\hat{\eta}_N(X)-\bar{\eta} (X)|\rightarrow0$.
This means that if the number of samples in set $\mathcal{A}$ are large enough, the estimation error of samples tends to zero.

{\bf{Term 2. Bound on the approximation error of set $\mathcal{A}$.}}
Let $\textbf{\textit{m}}(x|x\in\mathcal{B})$ and $\textbf{\textit{m}}(x|x\in\mathcal{A})$ be the probability density function of samples in
$\mathcal{A}$ and $\mathcal{B}$, respectively.
According to the definition of expectation,

\begin{align}
&\mathbb{E}_{x\in\mathcal{A}}\mathbb{P}(x\in\mathcal{B}|X=x)\notag\\
=&\int\mathbb{P}(x\in\mathcal{B}|X=x) d\textbf{\textit{m}}(x|x\in\mathcal{A})\notag\\
=&\mathbb{E}_{x\in\mathcal{B}}\mathbb{P}(x\in\mathcal{B}|X=x)\frac{\textbf{\textit{m}}(x|x\in\mathcal{A})}{\textbf{\textit{m}}(x|x\in\mathcal{B})}.
\end{align}

Then, the difference between the probability's expectation over $\mathcal{A}$ and $\mathcal{B}$ is:
\begin{align}\label{term2}
&|\bar{\eta}(X)-\bar{\eta}_B(X)| \notag \\ 
=& |\mathbb{E}_{x\in\mathcal{A}}\mathbb{P}(x\in\mathcal{B}|X=x) \notag \\ 
 & - \mathbb{E}_{x\in\mathcal{B}}\mathbb{P}(x\in\mathcal{B}|X=x)|  \notag \\ 
= &\bigg|\mathbb{E}_{x\in\mathcal{B}}\mathbb{P}(x\in\mathcal{B}|X=x)
\bigg( \frac{\textbf{\textit{m}}(x|x\in\mathcal{A})}{\textbf{\textit{m}}(x|x\in\mathcal{B})}-1 \bigg)
\bigg| \notag\\ 
\leq &
\sqrt{
\mathbb{E}_{x\in\mathcal{B}} \big( \mathbb{P}(x\in\mathcal{B}|X=x)\big)^2
} \notag \\
& \times \sqrt{
\mathbb{E}_{x\in\mathcal{B}}  \bigg( \frac{\textbf{\textit{m}}(x|x\in\mathcal{A})}{\textbf{\textit{m}}(x|x\in\mathcal{B})}-1 \bigg)^2
} 
\end{align}
where the last inequality is according to the Cauchy-Schwarz inequality. That is, for any random variables $X,Y$, there is $\mathbb{E}|XY|\le\sqrt{\mathbb{E}X^2\mathbb{E}Y^2}$.

Finally, by taking the logarithm and exponent on the right side of the inequality, we obtain
\begin{align}
&|\bar{\eta}(X)-\bar{\eta}_B(X)|\notag\\
\leq & \exp \bigg\{
\frac{1}{2}\log \mathbb{E}_{x\in\mathcal{B}} \big( \mathbb{P}(x\in\mathcal{B}|X=x)\big)^2  \notag\\
&\quad +\frac{1}{2} \log
\mathbb{E}_{x\in\mathcal{B}}
\bigg( \frac{\textbf{\textit{m}}(x|x\in\mathcal{A})}{\textbf{\textit{m}}(x|x\in\mathcal{B})}-1\bigg)^2
\bigg\}.
\end{align}

From the above inequality, 
with the assumption that the second moment of $\eta(X)$ of $\mathcal{B}$ exists, 
we obtain that when $ \mathbb{E}_{x\in\mathcal{B}}
\big( \frac{\textbf{\textit{m}}(x|x\in\mathcal{A})}{\textbf{\textit{m}}(x|x\in\mathcal{B})}-1\big)^2\rightarrow0$, the difference $|\bar{\eta}(X)-\bar{\eta}_B(X)|\rightarrow0$.
This means that if the probability distribution functions of set $\mathcal{A}$ and set $\mathcal{B}$ approach the same, the approximation error of $\mathcal{A}$ to $\mathcal{B}$ tends to zero.

{\bf{Term 3. Bound on the absolute deviation.}}

For the third term in Formula (\ref{errorbound_A}), according to the Cauchy-Schwarz inequality, we have:
\begin{align}
&\mathbb{E}_ {x\in\mathcal{X}}|\bar{\eta}_B(X)- \eta (X)|\notag \\
=&\mathbb{E}_{x\in\mathcal{B}}  |\bar{\eta}_B(X)- \eta (X)|
\frac{\textbf{\textit{m}}(x|x\in\mathcal{X})}{\textbf{\textit{m}}(x|x\in\mathcal{B})}\notag \\
\le&
\sqrt{
\mathbb{E}_{x\in\mathcal{B}}  |\bar{\eta}_B(X)- \eta (X)|^2
}
\sqrt{
\mathbb{E}_{x\in\mathcal{B}}\bigg(\frac{\textbf{\textit{m}}(x|x\in\mathcal{X})}{\textbf{\textit{m}}(x|x\in\mathcal{B})}\bigg)^2
}\notag \\
=&\sqrt{
\mathbb{V}_{x\in\mathcal{B}} (\eta (X))
}
\sqrt{
\mathbb{E}_{x\in\mathcal{B}}\bigg(\frac{\textbf{\textit{m}}(x|x\in\mathcal{X})}{\textbf{\textit{m}}(x|x\in\mathcal{B})}\bigg)^2
},
\end{align}
where $\mathbb{V}$ is the variance operator.
From the above inequality, with the assumption that the difference
in probability density between set $\mathcal{B}$ and the whole set $\mathcal{X}$ is bounded,
we obtain that when $\mathbb{V}_{x\in\mathcal{B}} (\eta (X))\rightarrow0$, the difference
$\mathbb{E}_ {x\in\mathcal{X}}|\bar{\eta}_B(X)- \eta (X)|\rightarrow0$. This means that if the probabilities of the samples belonging to the set $\mathcal{B}$ are more concentrated around their expected value, it is more smaller that the expected absolute deviation between the posterior probability and the mean of the posterior probabilities over $\mathcal{B}$.
\end{proof}

\subsubsection{A.4.2 Some theoretical results on binomial distribution random variables}

In this section, we survey the binomial distribution random variables to give some theoretical results about its probability and expectation. Theses results are used in proving the Term 1 of Theorem \ref{G-E1}.

Specially, Lemma \ref{binomial0} gives a relation of the partial sum of the binomial variable with its probability. Lemma \ref{supP} gives the upper bound of the probability of a binomial random variable in terms of its value.

Based on Lemma \ref{binomial0} and Lemma \ref{supP}, Theorem  \ref{binomial} gives the upper bound on the difference between the binomial random variable and its expectation.

\begin{lemma}\cite{10.1111/rssa.12292}
\label{binomial0}
If $X$ follows the binomial distribution with parameter $(N,p)$, where $N$ is the number of trails and $p$ is the probability of success, then
\begin{align}
&\sum_{i=n}^NiC_N^ip^i(1-p)^{N-i} \\ \notag
=& Np\mathbb{P}(X\geq n)+n(1-p)\mathbb{P}(X= n).
\end{align}
\label{relation}
 \end{lemma}

\begin{proof} This is an existing lemma. For completeness, we provide the proof process as follows:
 \begin{align}
 &\sum_{i=n}^NiC_N^ip^i(1-p)^{N-i}-Np\mathbb{P}(X\geq n)\notag\\
 =&\sum_{i=n}^N (i-Np)C_N^ip^i(1-p)^{N-i}\notag\\
  =& \sum_{i=n}^N \big(i(1-p)-(N-i)p\big)C_N^ip^i(1-p)^{N-i} \notag\\
 =&  \sum_{i=n}^N  N  C_{N-1}^{i-1} p^i(1-p)^{N-i+1}-N  C_{N-1}^{i} p^{i+1}(1-p)^{N-i} \notag\\
  =& N C_{N-1}^{n-1} p^n(1-p)^{N-n+1} \notag\\
   = &  nC_{N}^{n} p^n(1-p)^{N-n+1}\notag\\
 =& n(1-p)\mathbb{P}(X= n).
 \end{align}
 \end{proof}

\begin{lemma}
If $X$ follows the binomial distribution with parameter $(N,p)$, then
\begin{equation}
 \mathbb{P}(X=i)\leq \sqrt{\frac{N}{2\pi i(N-i)}}, \quad 0<i<N.
\end{equation}
\label{supP}
 \end{lemma}

\begin{proof}
When $p=\frac{i}{N}$, $\frac{\partial}{\partial p}  p^i(1-p)^{N-i}=0$, so
$p^i(1-p)^{N-i}$ reaches the maximum value at $p=\frac{i}{N}$. Thus, we have:
\begin{align}
 \mathbb{P}(X=i)&=  C_N^ip^i(1-p)^{N-i}\notag\\
 &\leq  C_N^i \bigg(\frac{i}{N}\bigg)^i\bigg(\frac{N-i}{N}\bigg)^{N-i} \notag\\
 &\leq \sqrt{\frac{N}{2\pi i(N-i)}}.
\end{align}
Then according to the Stirling formula,
\begin{equation}
C_N^i \leq \frac{N^N}{i^i(N-i)^{N-i} } \frac{1}{\sqrt{2\pi}}\sqrt {\frac{N}{i(N-i)}}.
\end{equation}
The lemma is proved.
 \end{proof}

 \begin{theorem}
 \label{binomial}
 If $X$ follows the binomial distribution with parameter $(N,p)$, suppose $Np$ is an integer, then
 \begin{equation}
\mathbb{E}|X-Np|
\leq \sqrt{\frac{2 Np(1-p)}{\pi }} \leq  \sqrt{\frac{N}{2\pi}}.
\end{equation}

 \end{theorem}
\begin{proof}
Because $\mathbb{E} (X-Np) =0$, we have:
\begin{align}
 &\sum_{i=0}^{ Np-1}(i-Np)C_N^ip^i(1-p)^{N-i} \notag\\
=&-\sum_{i= Np}^{N}(i-Np)C_N^ip^i(1-p)^{N-i}.
\end{align}
Then, according to Lemma \ref{relation} and Lemma \ref{supP}:
\begin{align}
&\mathbb{E}|X-Np|\notag\\
=& \sum_{i=0}^N|i-Np|C_N^ip^i(1-p)^{N-i} \notag\\
=& \sum_{i=0}^{ Np-1}(Np-i)C_N^ip^i(1-p)^{N-i} \notag\\
&+ \sum_{i= Np}^{N}(i-Np)C_N^ip^i(1-p)^{N-i}\notag\\
=& 2\sum_{i= Np}^{N}(i-Np)C_N^ip^i(1-p)^{N-i}\notag\\
=& 2 Np(1-p)\mathbb{P}(X= Np)\notag\\
= & 2 Np(1-p) \frac{1}{\sqrt{2\pi Np(1-p)}} \notag\\
= &\sqrt{ \frac{ 2{Np(1-p)}}{\pi }} \leq \sqrt{ \frac{ N }{2\pi }}.
\end{align}
\end{proof}

\subsection{A.5 The working mechanism of CEHM on the artificial data sets}

The information about the artificial data sets used in the paper is shown in Table \ref{artdata}. 

The working mechanism of CEHM on the last three artificial data sets in Table \ref{artdata} (2d2k, EngyTime, and Flame) are shown in Figure \ref{2d2k} to Figure \ref{Flame}.

\begin{table}[!ht]
\begin{center}
\caption{Description of the artificial data sets}
\begin{tabular*}{\hsize}{@{\extracolsep{\fill}}cccc}
\hline
Number & Data name &  n & k\\
\hline
1	&	WingNut 	&	1016		&	2	\\
2	&	2d2k 	&	1000		&	2	\\
3	&	EngyTime 	&	4096		&	2	\\
4	&	Flame 	&	240	&	2	\\
\hline
\end{tabular*}
\label{artdata}
\end{center}
\end{table}

\begin{figure}[!ht]
\captionsetup[subfigure]{labelformat=empty}
      \centering
       \subfloat[\scriptsize{\textbf{(a)} $\mathcal{E}i$}]{\includegraphics[width=0.15\textwidth]{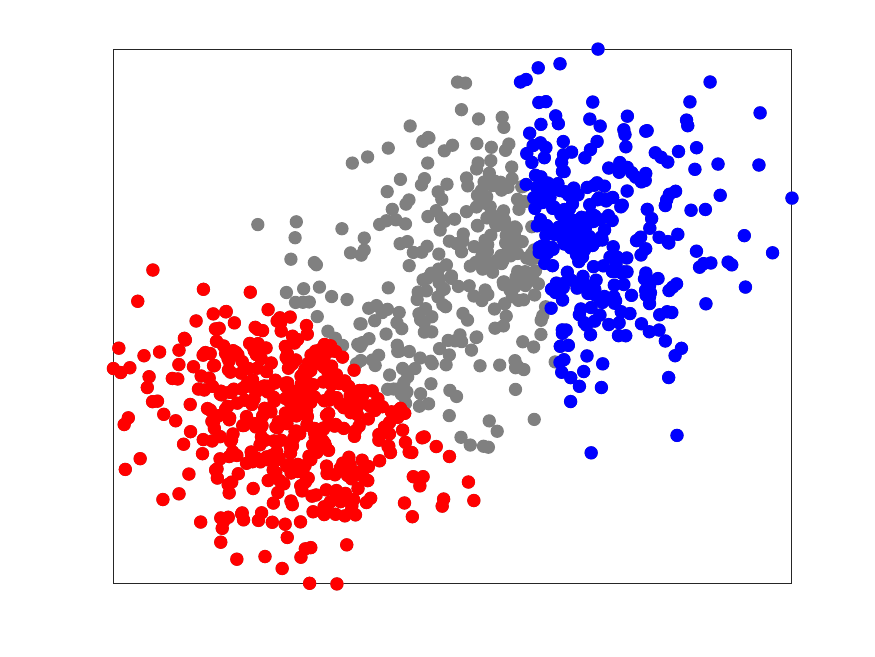}}\;
       \subfloat[\scriptsize{\textbf{(b)} $\mathcal{E}d(1)$}]{\includegraphics[width=0.15\textwidth]{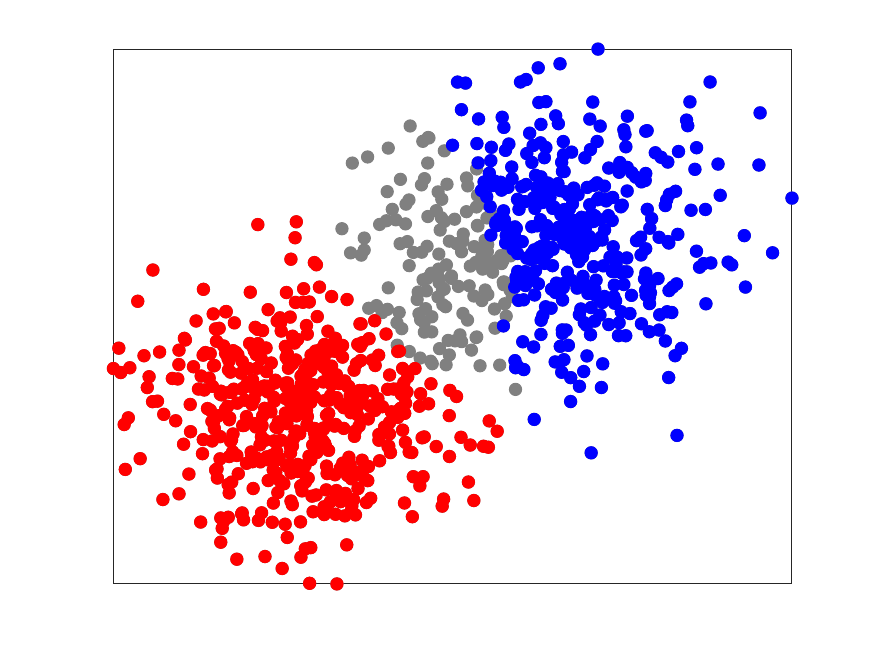}}\;
       \subfloat[\scriptsize{\textbf{(c)} $\mathcal{E}d(2)$}]{\includegraphics[width=0.15\textwidth]{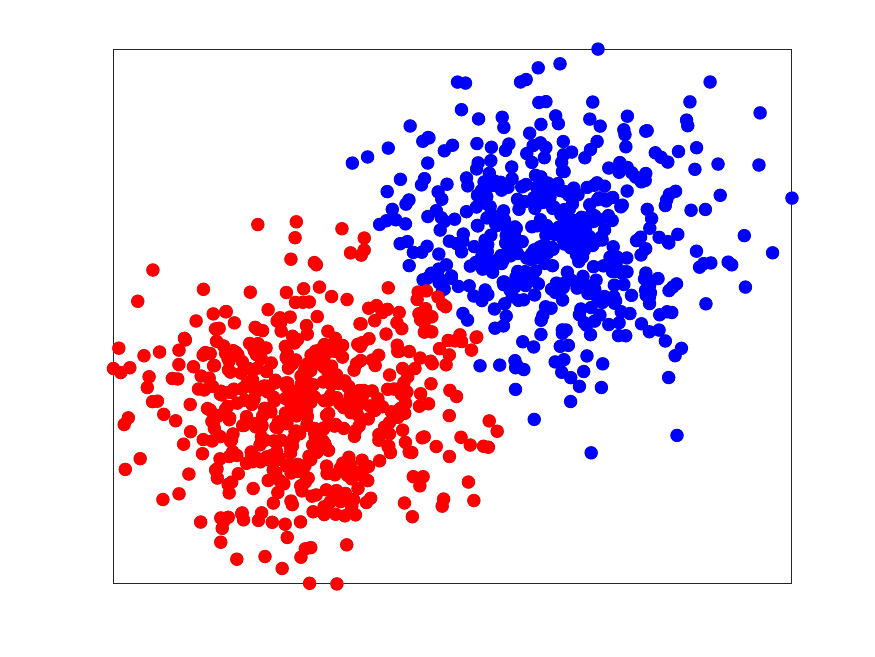}}\
       \subfloat[\scriptsize{\textbf{(d)} $\mathcal{E}a(1)$}]{\includegraphics[width=0.15\textwidth]{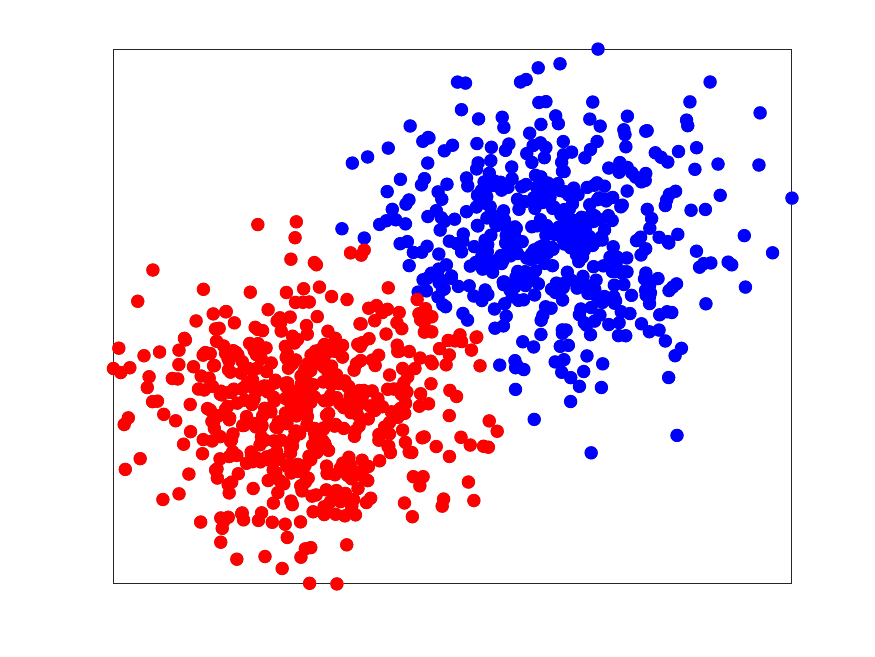}}\;
       \subfloat[\scriptsize{\textbf{(e)} $\mathcal{E}a(2)$}]{\includegraphics[width=0.15\textwidth]{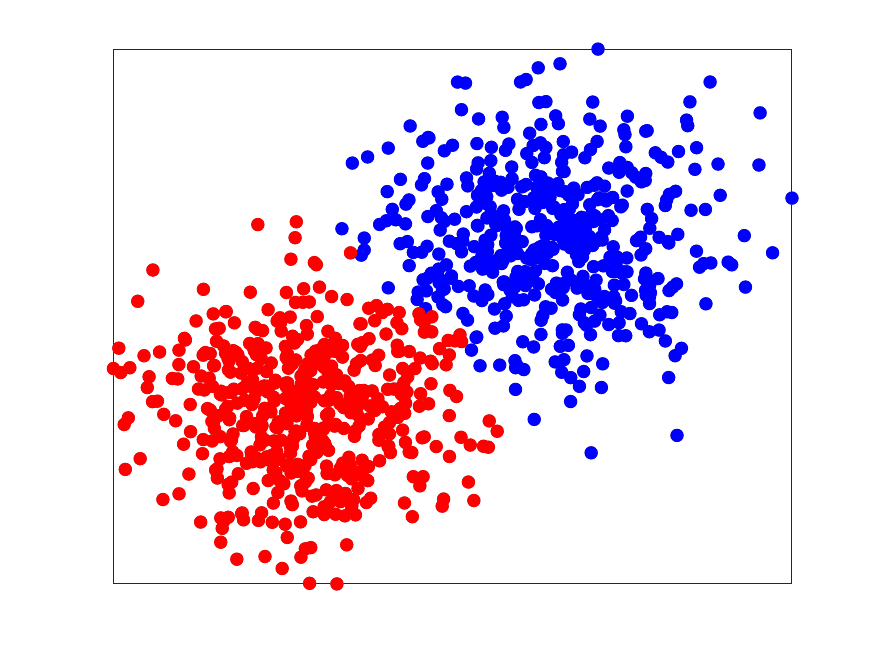}}\;
       \subfloat[\scriptsize{\textbf{(f)} $\mathcal{E}_m$}]{\includegraphics[width=0.15\textwidth]{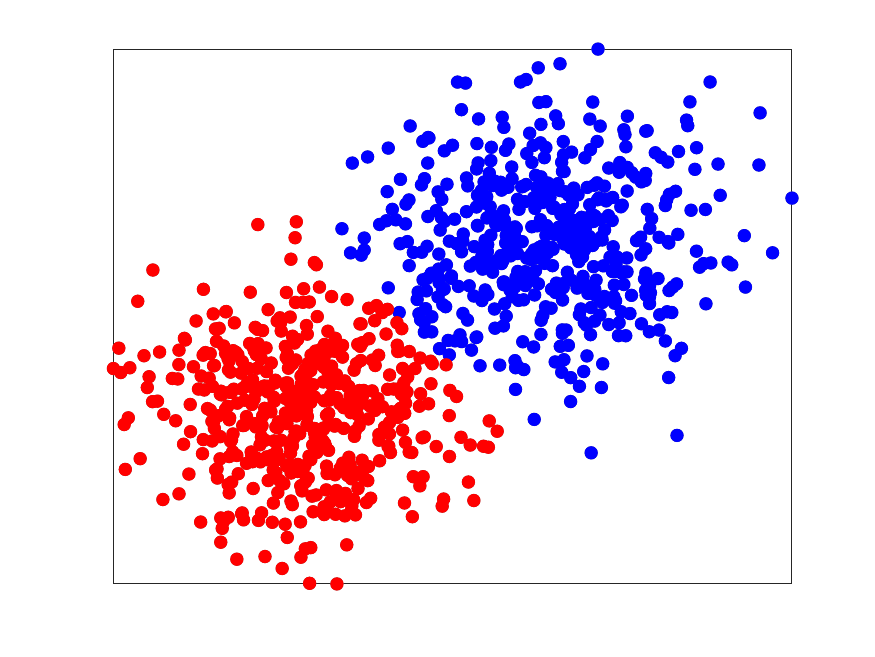}}
\caption{Working mechanism of CEHM on the 2d2k data set}
\label{2d2k}
\end{figure}

\begin{figure}[!ht]
\captionsetup[subfigure]{labelformat=empty}
      \centering
       \subfloat[\scriptsize{\textbf{(a)} $\mathcal{E}i$}]{\includegraphics[width=0.15\textwidth]{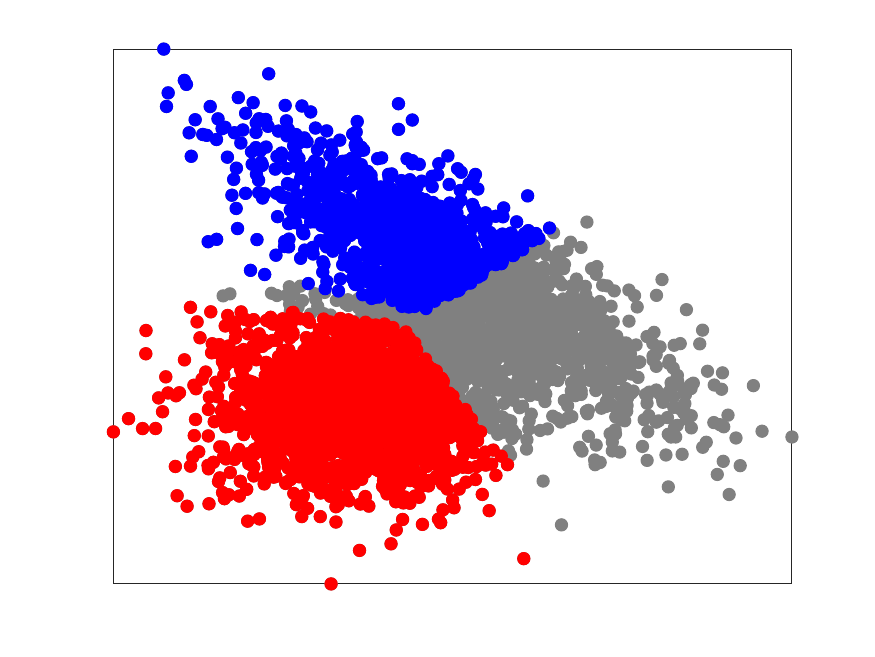}}\;
       \subfloat[\scriptsize{\textbf{(b)} $\mathcal{E}d(1)$}]{\includegraphics[width=0.15\textwidth]{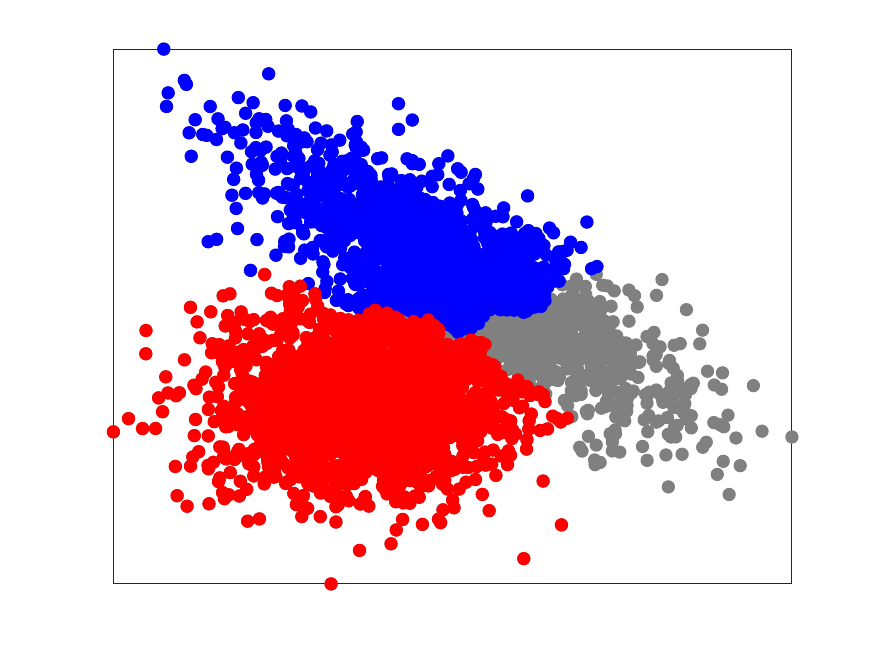}}\;
       \subfloat[\scriptsize{\textbf{(c)} $\mathcal{E}d(2)$}]{\includegraphics[width=0.15\textwidth]{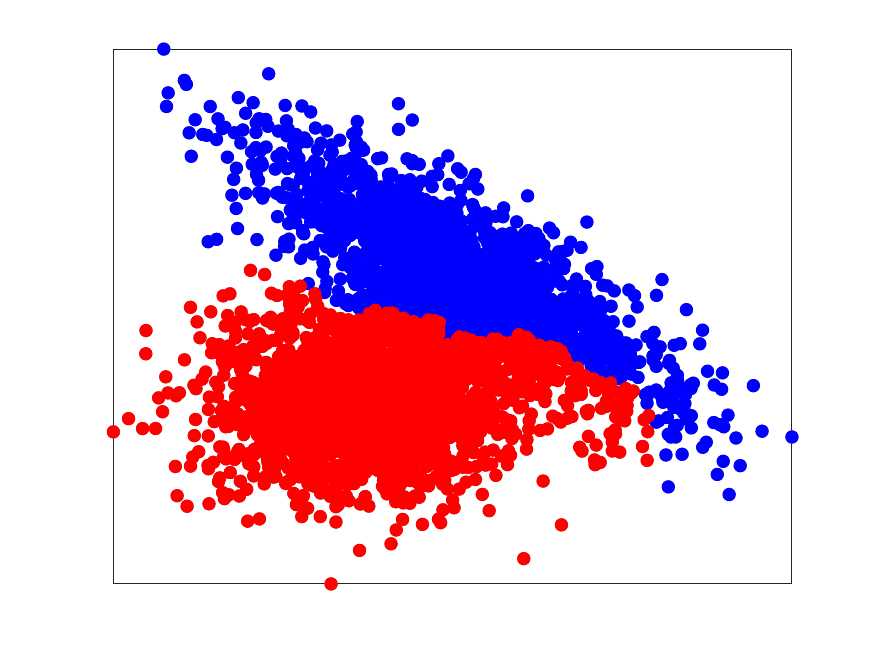}}\
       \subfloat[\scriptsize{\textbf{(d)} $\mathcal{E}a(1)$}]{\includegraphics[width=0.15\textwidth]{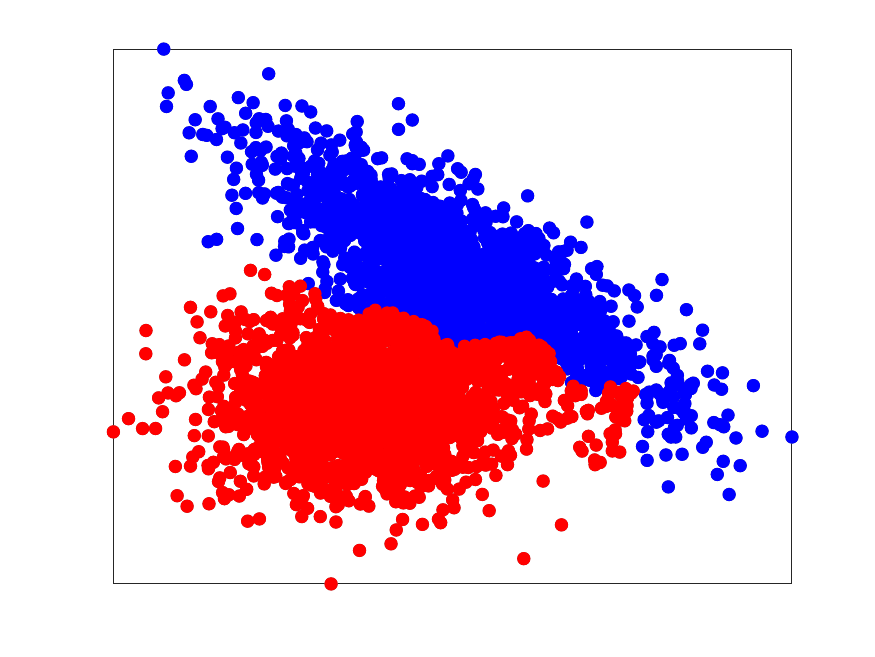}}\;
       \subfloat[\scriptsize{\textbf{(e)} $\mathcal{E}a(2)$}]{\includegraphics[width=0.15\textwidth]{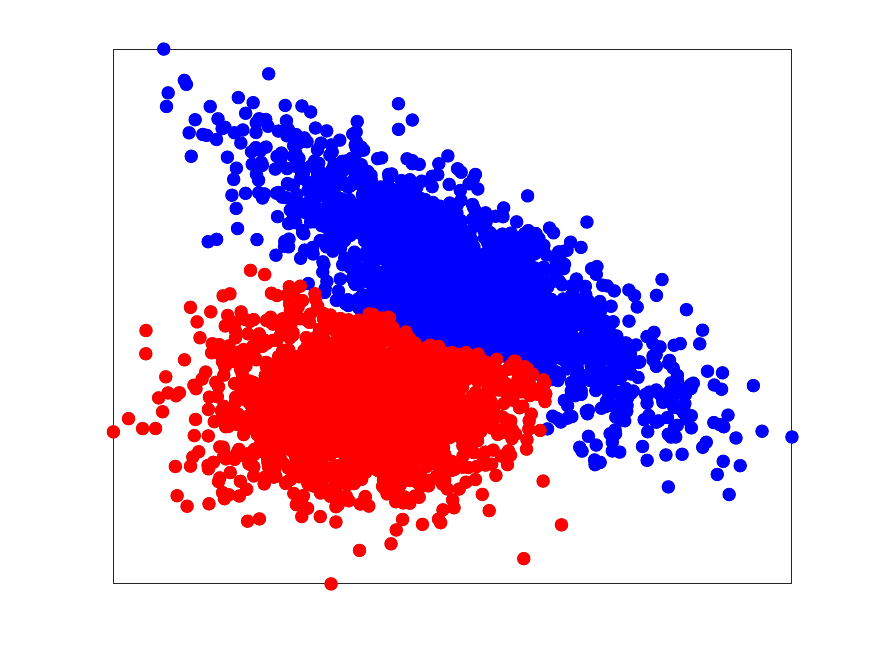}}\;
       \subfloat[\scriptsize{\textbf{(f)} $\mathcal{E}_m$}]{\includegraphics[width=0.15\textwidth]{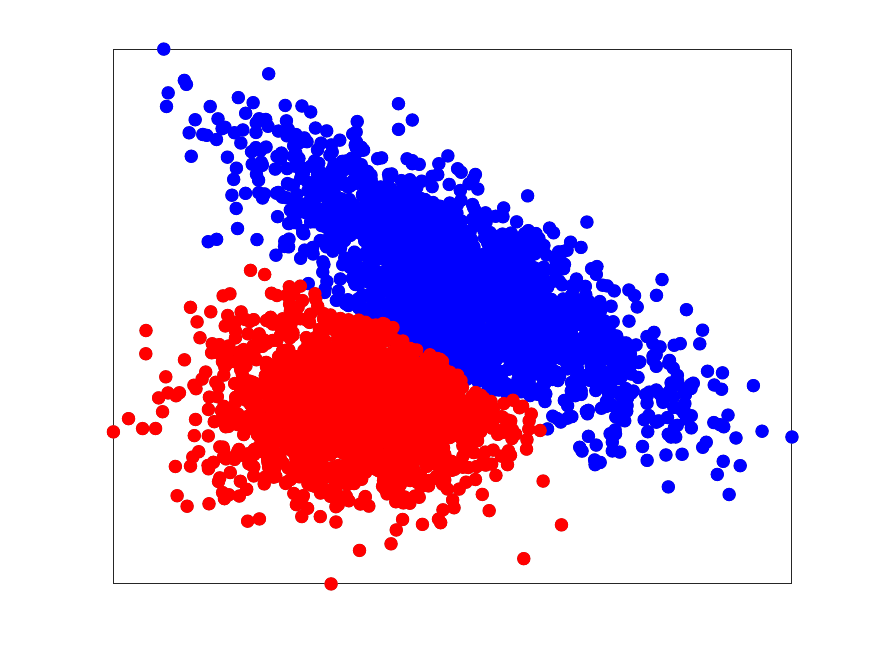}}
\caption{Working mechanism of CEHM on the EngyTime data set}
\label{EngyTime}
\end{figure}

\begin{figure}[!ht]
\captionsetup[subfigure]{labelformat=empty}
      \centering
       \subfloat[\scriptsize{\textbf{(a)} $\mathcal{E}i$}]{\includegraphics[width=0.15\textwidth]{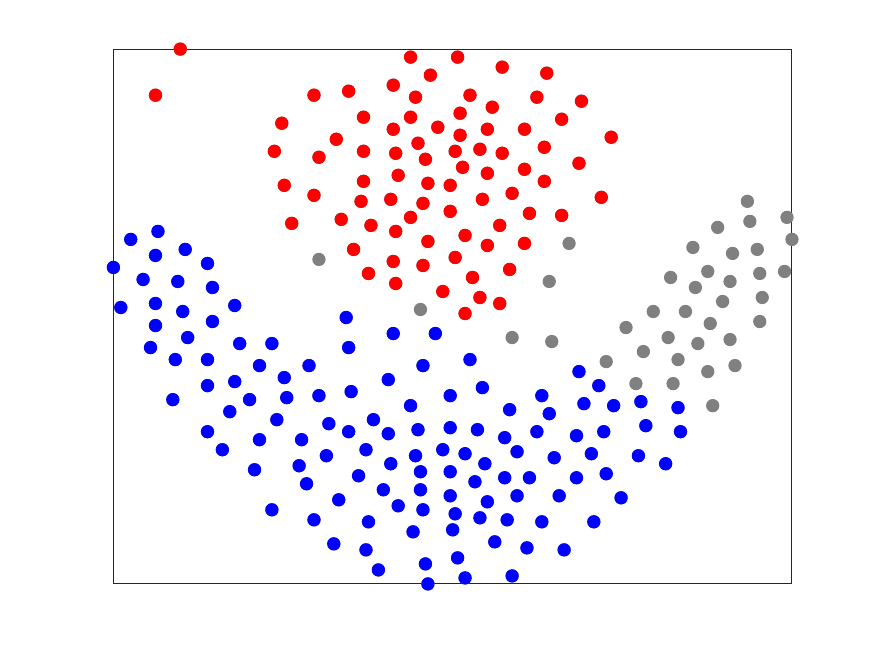}}\;
       \subfloat[\scriptsize{\textbf{(b)} $\mathcal{E}d(1)$}]{\includegraphics[width=0.15\textwidth]{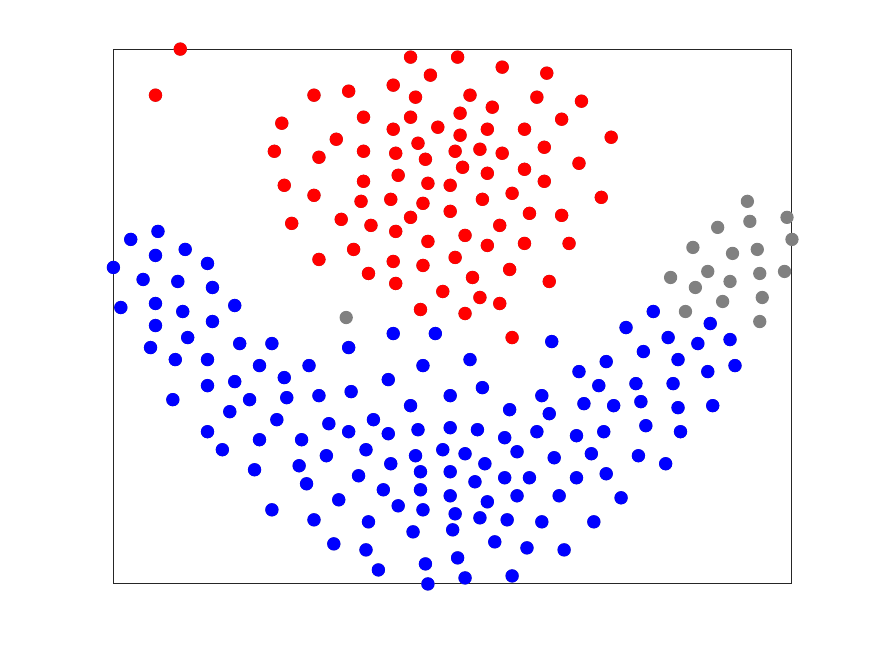}}\;
       \subfloat[\scriptsize{\textbf{(c)} $\mathcal{E}d(2)$}]{\includegraphics[width=0.15\textwidth]{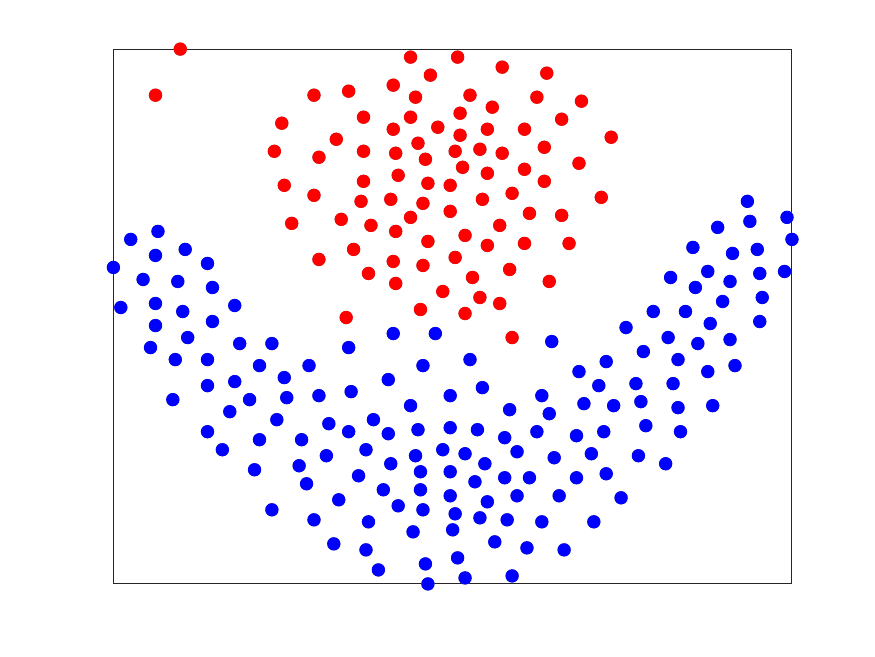}}\
       \subfloat[\scriptsize{\textbf{(d)} $\mathcal{E}a(1)$}]{\includegraphics[width=0.15\textwidth]{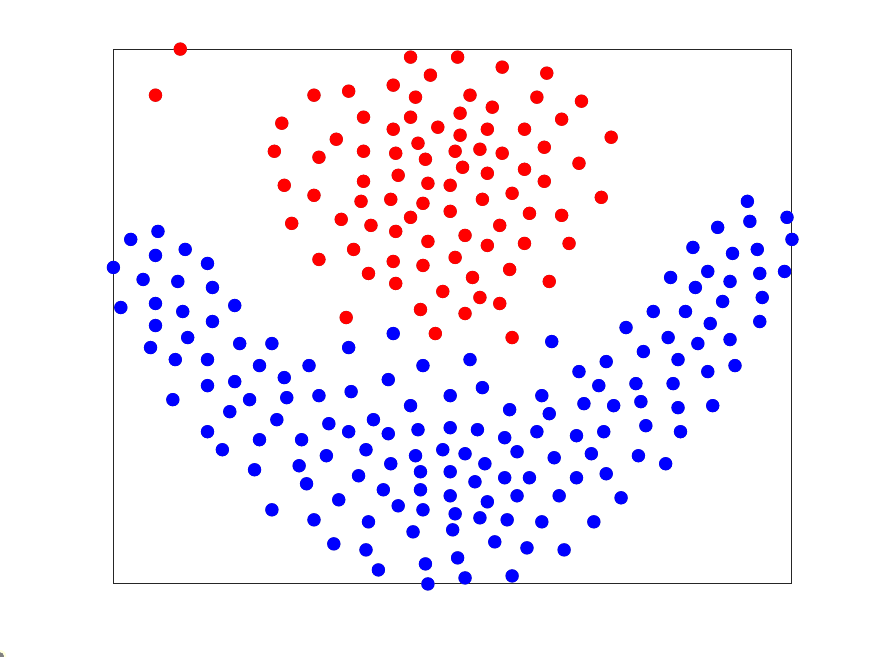}}\;
       \subfloat[\scriptsize{\textbf{(e)} $\mathcal{E}a(2)$}]{\includegraphics[width=0.15\textwidth]{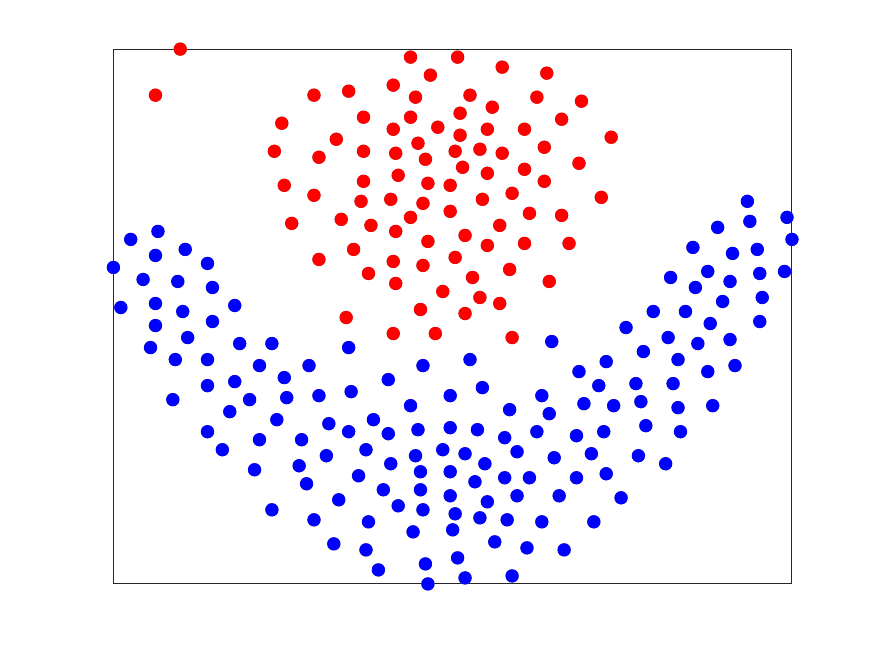}}\;
       \subfloat[\scriptsize{\textbf{(f)} $\mathcal{E}_m$}]{\includegraphics[width=0.15\textwidth]{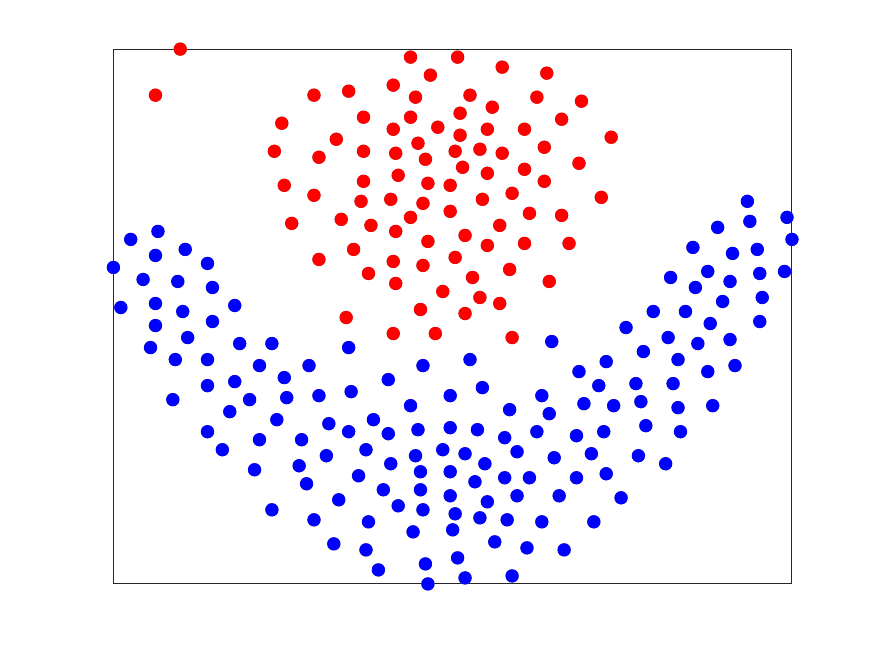}}
\caption{Working mechanism of CEHM on the Flame data set}
\label{Flame}
\end{figure}

\subsection{A.6 The convergence results on the last fifteen data sets}

The convergence analysis on the last fifteen data sets in Table \ref{ucidata} in the main paper is shown in Figure \ref{conve_whole}. From Figure \ref{conve_whole}, it is easy to see that the value of the loss function is gradually reduced. In addition, on these data sets, the maximum number of iterations is not greater than 30.

\begin{figure}[!ht]
\captionsetup[subfigure]{labelformat=empty}
      \centering
       \subfloat[\scriptsize{\textbf{(f)} data 6}]{\includegraphics[width=0.15\textwidth]{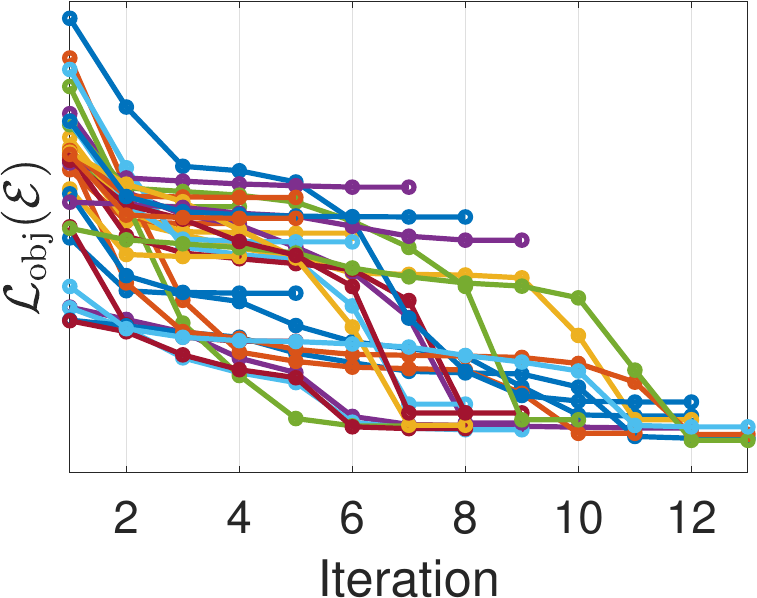}}\;
       \subfloat[\scriptsize{\textbf{(g)} data 7}]{\includegraphics[width=0.15\textwidth]{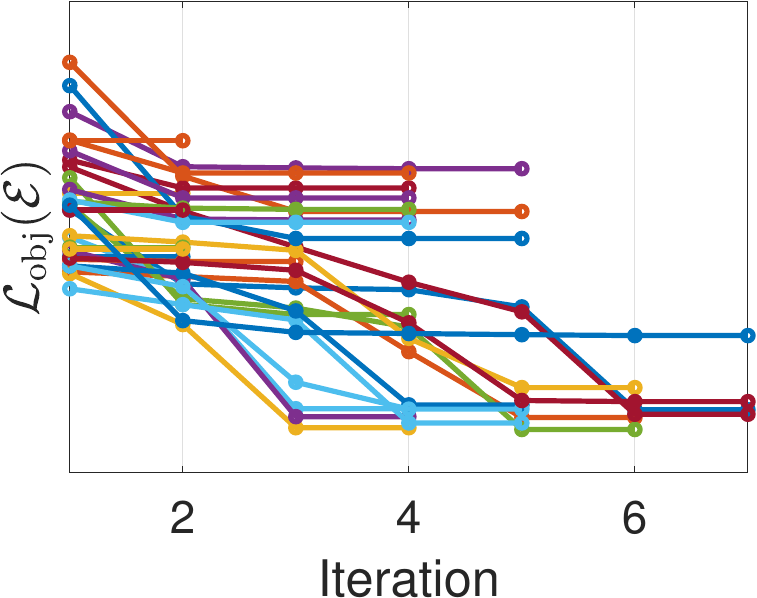}}\;
       \subfloat[\scriptsize{\textbf{(h)} data 8}]{\includegraphics[width=0.15\textwidth]{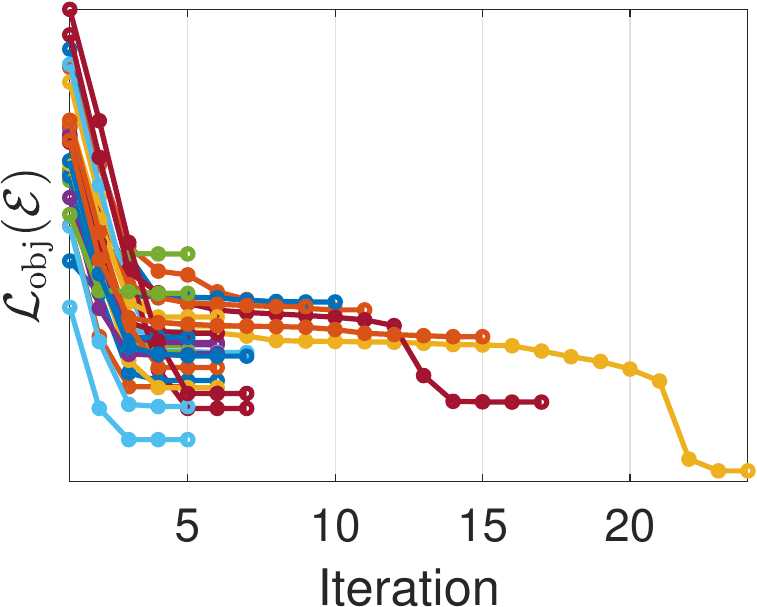}}\
       \subfloat[\scriptsize{\textbf{(i)} data 9}]{\includegraphics[width=0.15\textwidth]{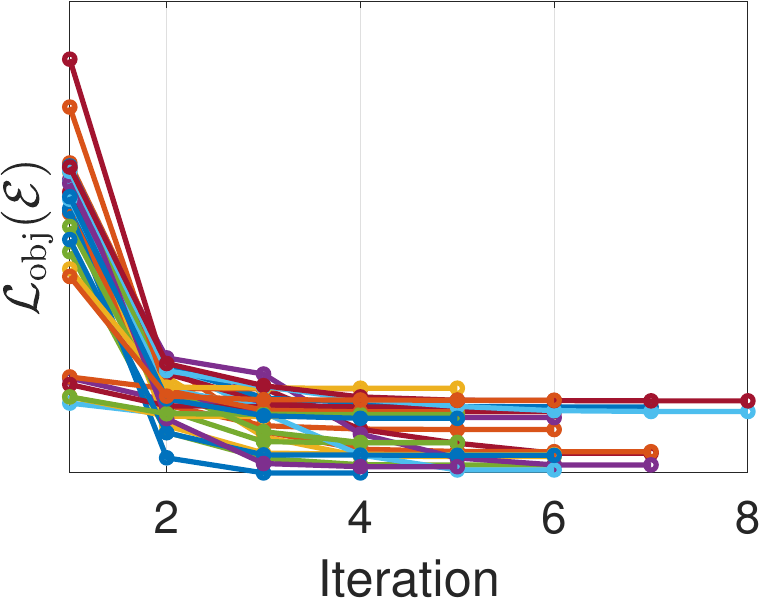}}\;
       \subfloat[\scriptsize{\textbf{(j)} data 10}]{\includegraphics[width=0.15\textwidth]{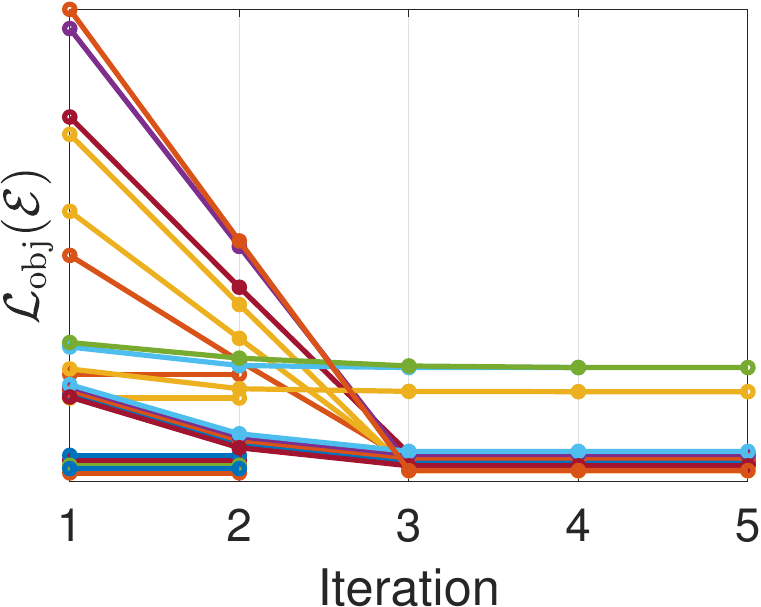}}\;
       \subfloat[\scriptsize{\textbf{(k)} data 11}]{\includegraphics[width=0.15\textwidth]{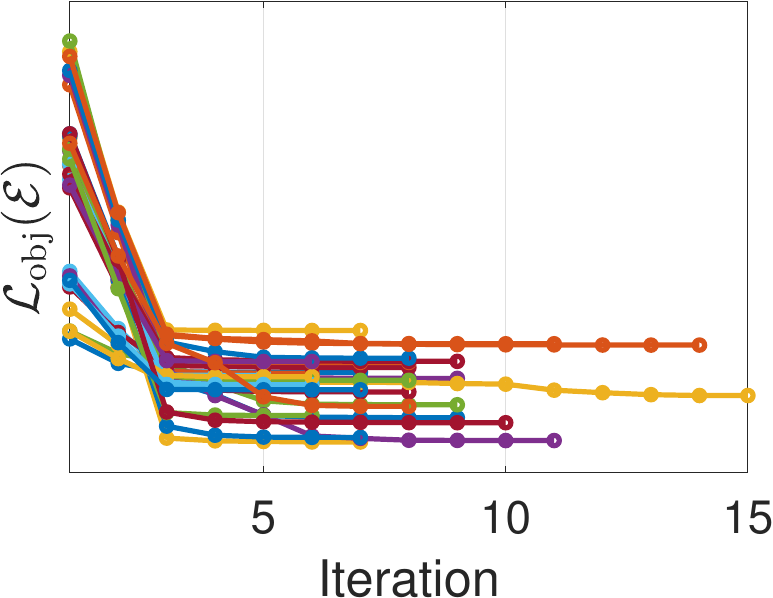}}\
       \subfloat[\scriptsize{\textbf{(l)} data 12}]{\includegraphics[width=0.15\textwidth]{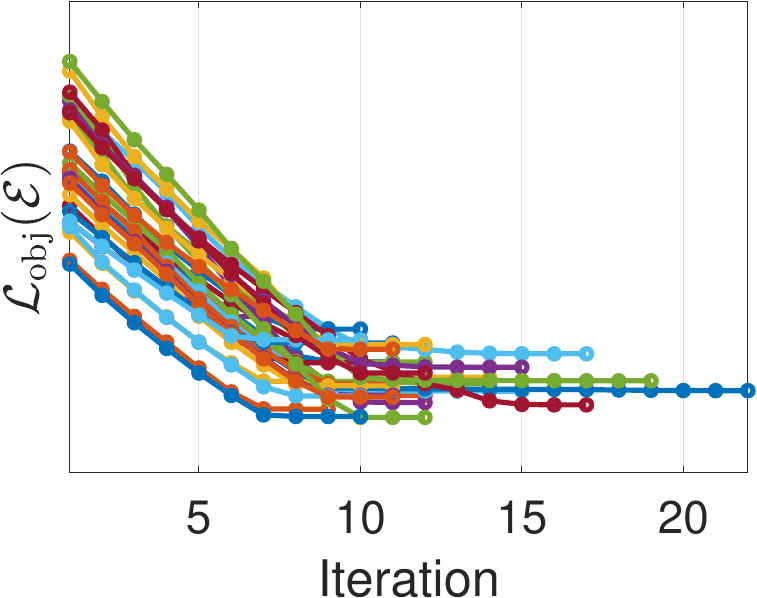}}\;
       \subfloat[\scriptsize{\textbf{(m)} data 13}]{\includegraphics[width=0.15\textwidth]{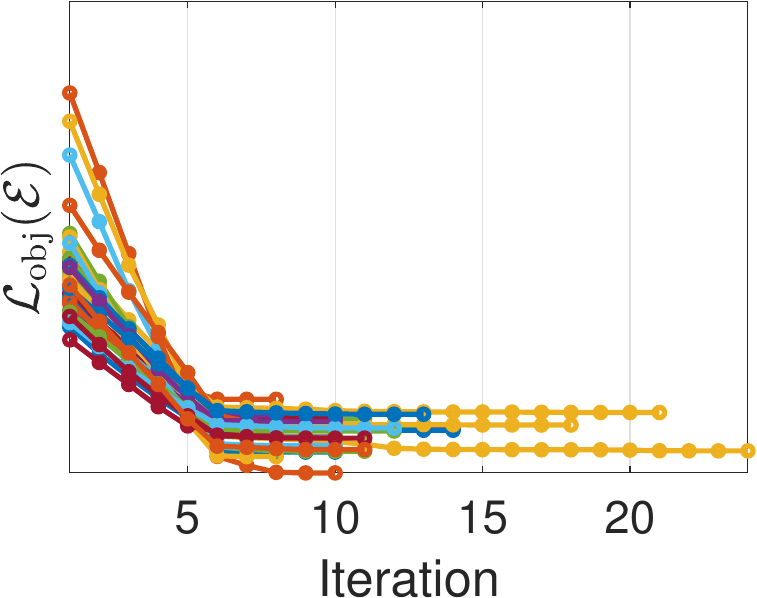}}\;
       \subfloat[\scriptsize{\textbf{(n)} data 14}]{\includegraphics[width=0.15\textwidth]{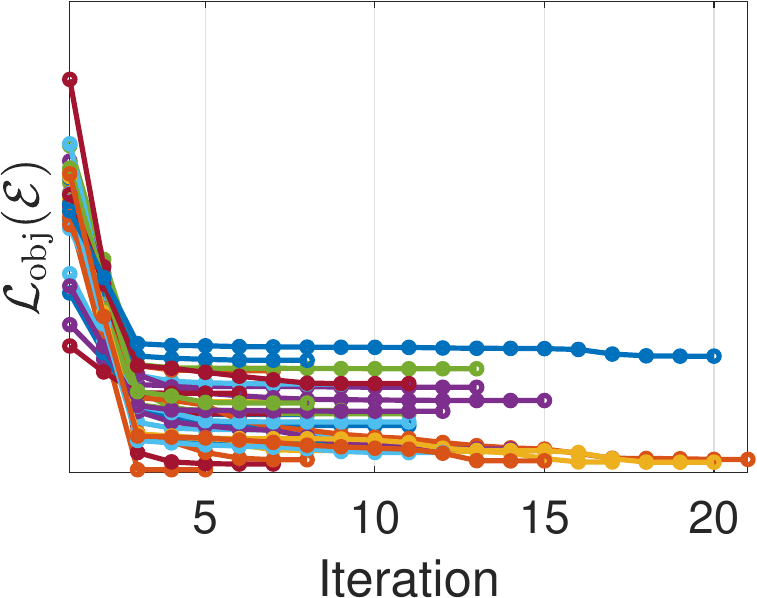}}\
       \subfloat[\scriptsize{\textbf{(o)} data 15}]{\includegraphics[width=0.15\textwidth]{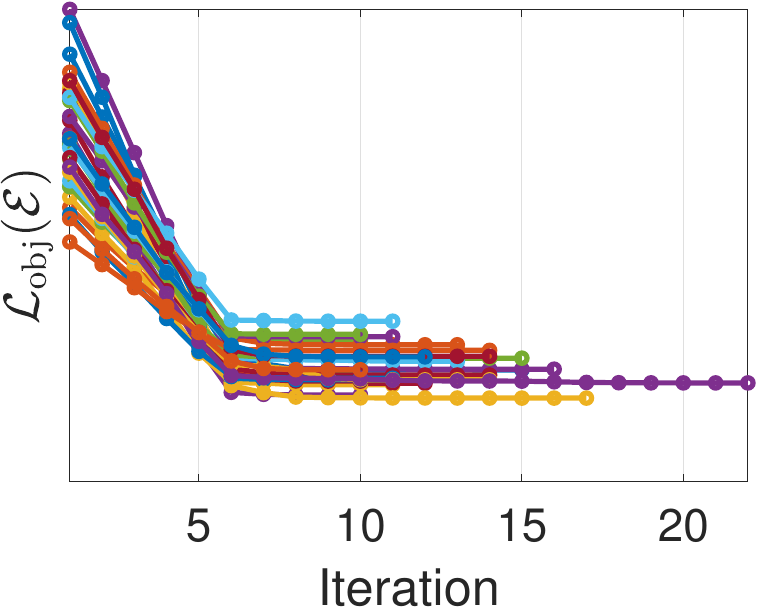}}\;
       \subfloat[\scriptsize{\textbf{(p)} data 16}]{\includegraphics[width=0.15\textwidth]{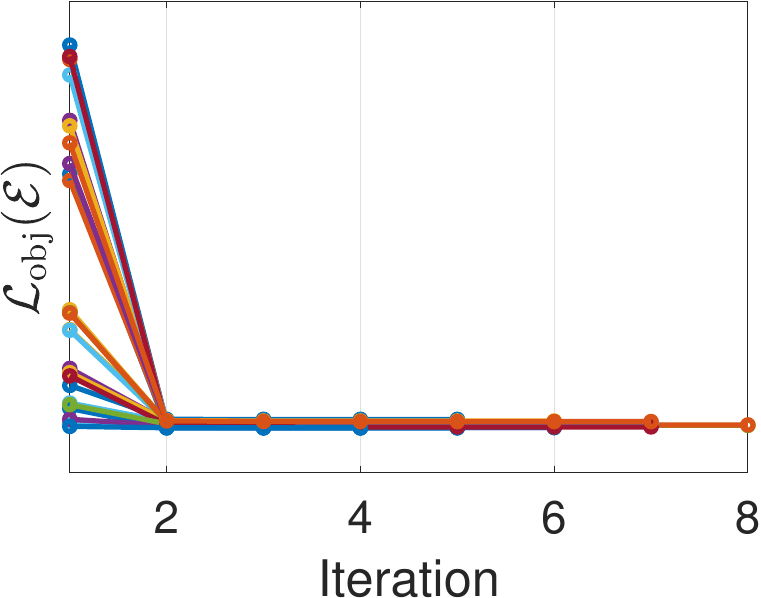}}\;
       \subfloat[\scriptsize{\textbf{(q)} data 17}]{\includegraphics[width=0.15\textwidth]{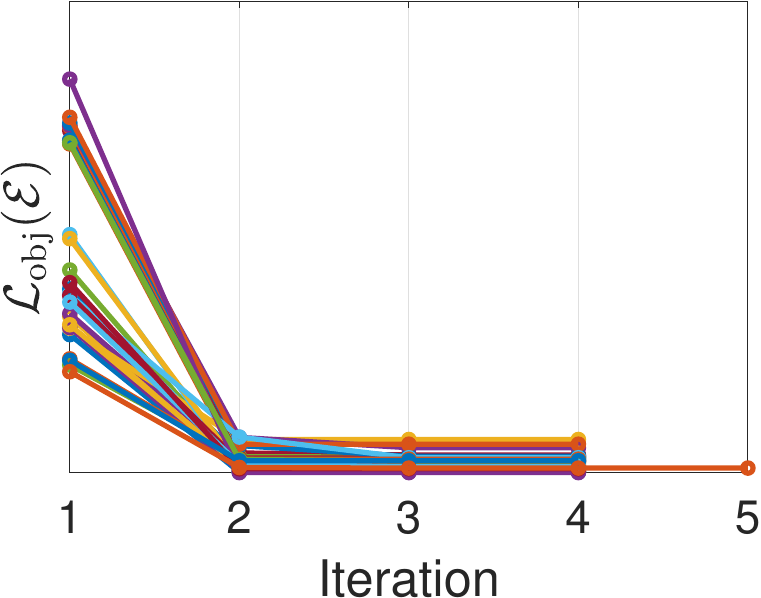}}\
       \subfloat[\scriptsize{\textbf{(r)} data 18}]{\includegraphics[width=0.15\textwidth]{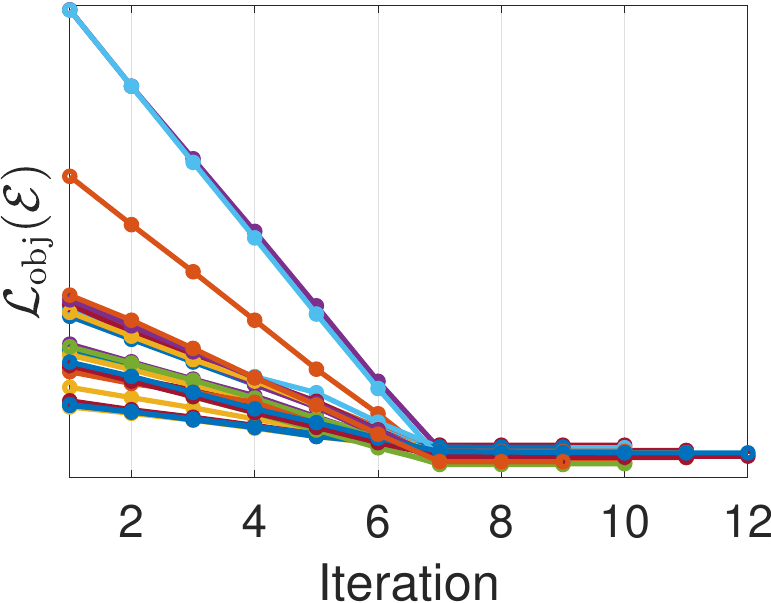}}\;
       \subfloat[\scriptsize{\textbf{(s)} data 19}]{\includegraphics[width=0.15\textwidth]{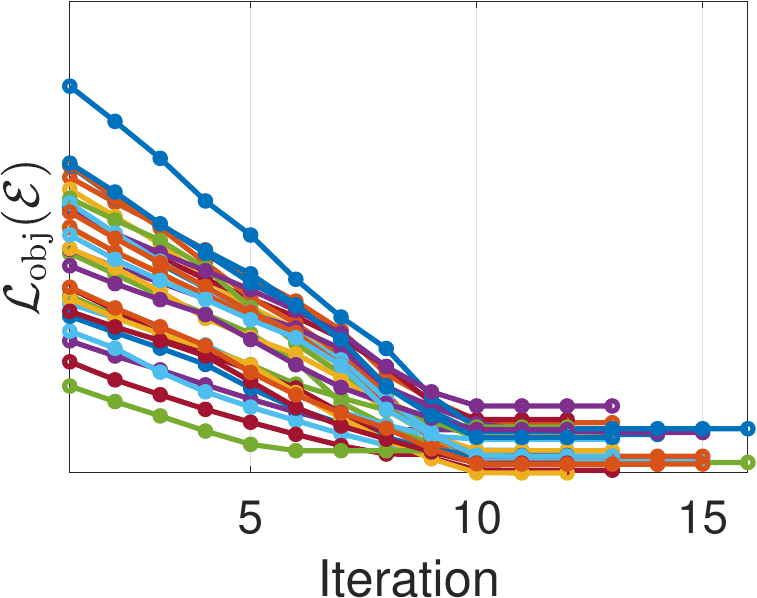}}\;
       \subfloat[\scriptsize{\textbf{(t)} data 20}]{\includegraphics[width=0.15\textwidth]{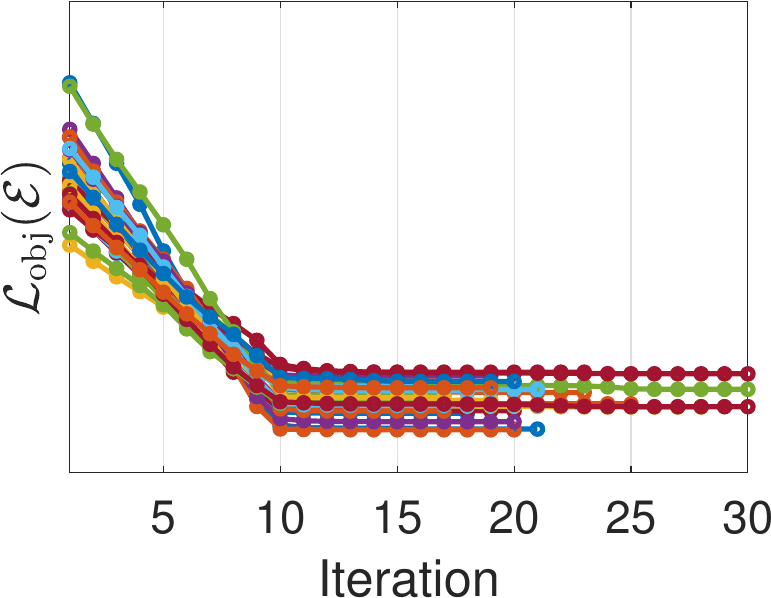}}
\caption{The convergence demonstration of CEHM on the last fifteen data sets}
\label{conve_whole}
\end{figure}

\subsection{A.7 The compared methods and their settings} \label{ce:re}

The compared clustering ensemble methods and their settings are as follows:
\begin{itemize}
  \item U-SENC \cite{Huang2020}: Ultra-Scalable Ensemble Clustering (2020).
  \item DREC \cite{Zhou2019Ensemble}: Ensemble clustering based on Dense Representation (2019). In DREC, the parameter setting is set as $\lambda=100$ as suggested.
  \item EC-CMS \cite{Jia2023}: Ensemble Clustering via Co-Association Matrix Self-Enhancement. In EC-CMS, as suggested, $\alpha=0.8$ and $\lambda=0.4$.
  \item PTGP \cite{huang2016robust}: Probability Trajectory based Graph Partitioning (2016). In PTGP, the parameters $T$ and $K$ take the settings in the published code.
  \item ECPCS-MC \cite{huang2021}: Ensemble Clustering by Propagating Cluster-wise Similarities with Meta-cluster-based Consensus function (2021). In ECPCS-MC, parameter $t=20$.
  \item PTA \cite{huang2016robust}: Probability Trajectory Accumulation (2016). In PTA, the parameters take the settings in the published code.
  \item WCT \cite{Iam2011}: Weighted Connected-Triple. In WCT, the parameter DC is set as 0.9 as suggested.
  \item WTQ \cite{Iam2011}: Weighted Triple-Quality. In WTQ, the parameter DC is set as 0.9 as suggested.
  \item Voting \cite{Li2017}: Voting based on Meta-CLustering Algorithm (2017).
\end{itemize}

\begin{table*}[!th]
\begin{center}
\caption{The ARI from the compared methods on the twenty data sets}
\small
\begin{tabular*}{\hsize}{@{\extracolsep{\fill}}ccccccccccc}
\hline
Data& U-SENC & DREC & EC-CMS & PTGP & ECPCS-MC & PTA & WCT & WTQ & Voting & CEHM\\
\hline
1&0.6256&0.6300&0.6176&0.6266&0.6372&0.4569&0.6204&0.6081&0.5757&\textbf{\underline{0.6553}}\\
2&0.9187&0.9076&0.8563&0.9046&0.8998&0.8886&0.9176&0.8677&0.8409&\textbf{\underline{0.9261}}\\
3&0.7750&0.6235&0.6958&0.6312&0.7221&0.5504&0.7053&0.6821&0.7765&\textbf{\underline{0.7783}}\\
4&0.4194&0.3638&0.4136&0.2014&0.3876&0.1514&0.3778&0.2705&0.2100&\textbf{\underline{0.4210}}\\
5&\textbf{\underline{0.4519}}&0.4330&0.4423&0.4115&0.4462&0.4191&0.4315&0.4181&0.4116&0.4496\\
6&0.4281&0.4573&0.6129&0.3717&0.5792&0.4675&0.4596&0.4582&0.3200&\textbf{\underline{0.7290}}\\
7&0.8359&0.8127&0.7318&0.8479&0.8687&0.8503&0.7463&0.7208&0.7021&\textbf{\underline{0.8734}}\\
8&0.2405&0.2303&0.3894&0.2262&0.3536&0.2324&0.2292&0.2196&0.1772&\textbf{\underline{0.3950}}\\
9&0.7242&0.6943&0.2756&0.6534&0.6513&0.5832&0.6854&0.7043&0.5103&\textbf{\underline{0.7400}}\\
10&\textbf{\underline{0.6220}}&\textbf{\underline{0.6220}}&\textbf{\underline{0.6220}}&\textbf{\underline{0.6220}}&0.6217&\textbf{\underline{0.6220}}&0.5254&0.5514&0.3399&\textbf{\underline{0.6220}}\\
11&0.3487&0.3753&0.2457&0.1931&0.2648&0.3134&0.3056&0.2931&0.2437&\textbf{\underline{0.5174}}\\
12&0.5110&\textbf{\underline{0.5294}}&0.4679&0.5078&0.5142&0.5277&0.4771&0.4296&0.4890&0.5253\\
13&0.4977&0.5336&0.3438&0.4884&0.5009&0.4655&0.5175&0.4632&0.4782&\textbf{\underline{0.5452}}\\
14&0.2011&0.0549&0.1109&0.0249&0.2215&0.0548&0.0609&0.1417&0.0406&\textbf{\underline{0.2651}}\\
15&0.5508&0.5905&0.4361&0.5693&0.6063&0.6089&0.5399&0.4623&0.4383&\textbf{\underline{0.6266}}\\
16&0.9079&0.9066&OM&0.9057&0.9098&0.8913&0.9086&0.9080&0.9105&\textbf{\underline{0.9111}}\\
17&0.1857&0.2864&OM&0.2837&0.2000&0.0329&0.1381&0.2747&0.2950&\textbf{\underline{0.3759}}\\
18&0.3773&0.2506&OM&0.1573&0.0503&0.1829&0.1550&0.1258&0.0832&\textbf{\underline{0.5831}}\\
19&0.6325&OM&OM&OM&0.6233&OM&0.5829&0.5238&0.5405&\textbf{\underline{0.6392}}\\
20&0.3779&OM&OM&OM&0.3974&OM&0.4064&0.3268&0.3453&\textbf{\underline{0.4092}}\\
\hline
\end{tabular*}
\label{index_ari}
\end{center}
\end{table*}

\subsection{A.8 The ARI from the compared methods on the twenty data sets}

The ensemble performance results evaluated by ARI is shown in Table \ref{index_ari}. As shown by Table \ref{index_ari}, the CEHM is also marked on most of the data sets.

\subsection{A.9 Statistical analysis about the experimental results}

We statistically analyze the experimental results of Table \ref{index_nmi} and Table \ref{index_ari}. The analysis results are shown in Figure \ref{pdata2} to Figure \ref{pdata20}. The analysis was performed by the genescloud tools, a free online platform for data analysis (https://www.genescloud.cn).

Firstly, the Kruskal-Wallis Non-Parametric Test is utilized to test whether there are statistically significant differences between the distributions of the index values of the compared methods. As for the Kruskal-Wallis test, a low P-value (typically $p\leq 0.05$) generally indicates sufficient evidence to suggest that there is a statistically significant difference between at least one pair of medians. In Figure \ref{pdata1} to Figure \ref{pdata20}, the P-vale $p <0.0001$ for all the data sets and the two indices, which means there exists a significant difference between the compared methods.

Secondly, we further utilize Dunn's test to identify significant differences between pairs of methods. In Dunn's test, the P-value indicates the probability that the observed difference between two groups is due to random chance. A small P-value suggests a significant difference in the group. In Figure \ref{pdata2} to Figure \ref{pdata20}, the two ends of a line corresponding to the two compared methods, $p<0.05$ is represented by $*$, $p<0.01$ is represented by $**$, $p<0.0001$ is represented by $***$. From Figure \ref{pdata1} to Figure \ref{pdata20}, it can be seen that CEHM obtains many times of significant differences from other methods.

Finally, the boxplot with points is utilized to show the distributions of the indices values obtained by each compared method. It can be seen from Figure \ref{pdata1} to Figure \ref{pdata20} that CEHM obtains stable performance on most of those data sets.

\begin{figure}[!ht]
\captionsetup[subfigure]{labelformat=empty}
{\includegraphics[width=0.45\textwidth,height=0.2\textheight]{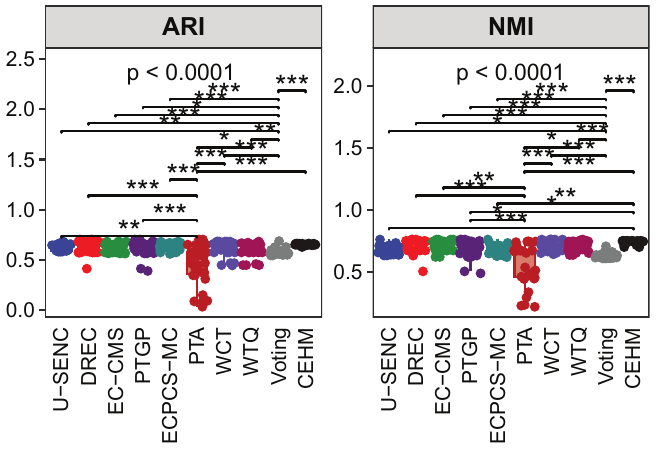}}\;
\caption{The Box Plot with Statistical Significance on data set 1}
\label{pdata1}
\end{figure}

\begin{figure}[!ht]
\captionsetup[subfigure]{labelformat=empty}
{\includegraphics[width=0.45\textwidth,height=0.2\textheight]{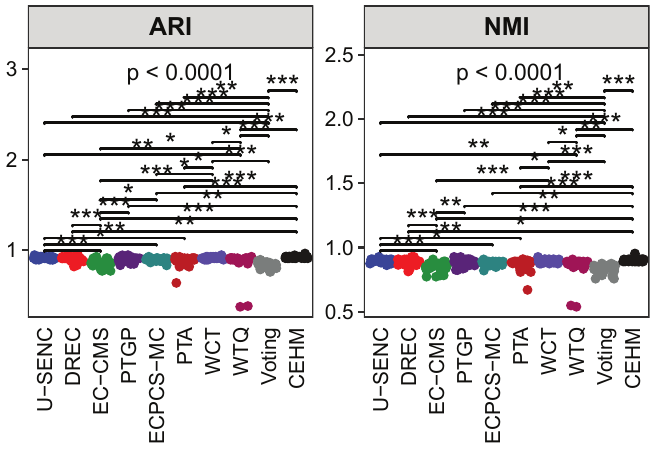}}\;
\caption{The Box Plot with Statistical Significance on data set 2}
\label{pdata2}
\end{figure}

\begin{figure}[!ht]
\captionsetup[subfigure]{labelformat=empty}
{\includegraphics[width=0.45\textwidth,height=0.2\textheight]{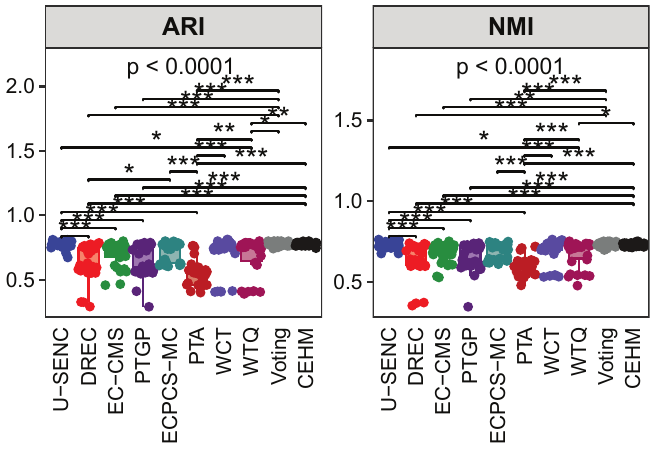}}\;
\caption{The Box Plot with Statistical Significance on data set 3}
\label{pdata3}
\end{figure}

\begin{figure}[!ht]
\captionsetup[subfigure]{labelformat=empty}
{\includegraphics[width=0.45\textwidth,height=0.2\textheight]{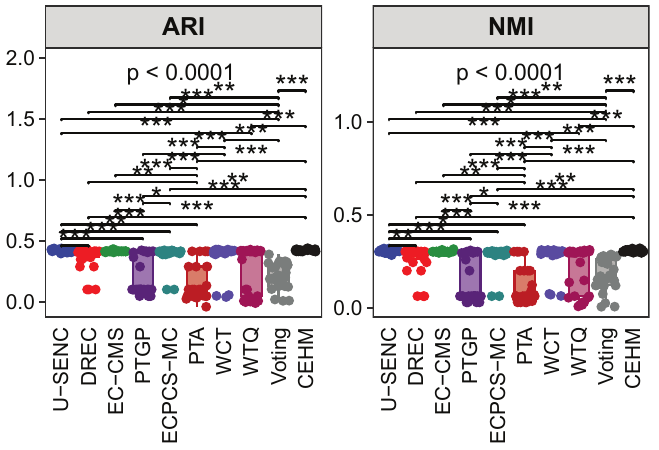}}\;
\caption{The Box Plot with Statistical Significance on data set 4}
\label{pdata4}
\end{figure}

\begin{figure}[!ht]
\captionsetup[subfigure]{labelformat=empty}
{\includegraphics[width=0.45\textwidth,height=0.2\textheight]{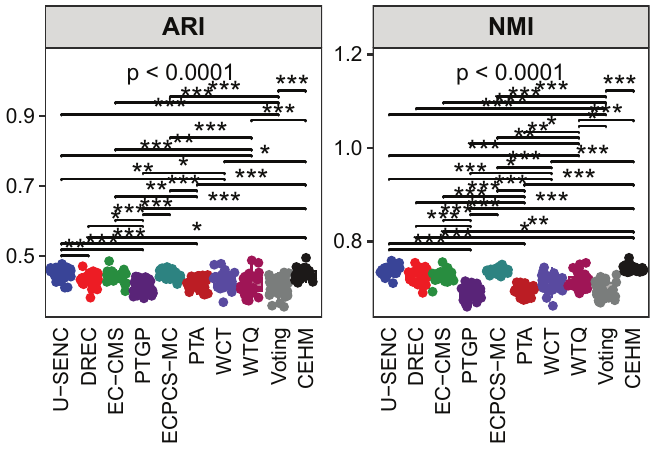}}\;
\caption{The Box Plot with Statistical Significance on data set 5}
\label{pdata5}
\end{figure}

\begin{figure}[!ht]
\captionsetup[subfigure]{labelformat=empty}
{\includegraphics[width=0.45\textwidth,height=0.2\textheight]{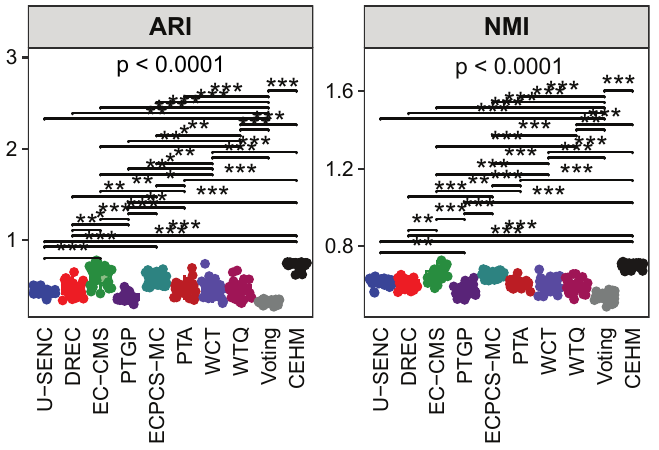}}\;
\caption{The Box Plot with Statistical Significance on data set 6}
\label{pdata6}
\end{figure}

\begin{figure}[!ht]
\captionsetup[subfigure]{labelformat=empty}
{\includegraphics[width=0.45\textwidth,height=0.2\textheight]{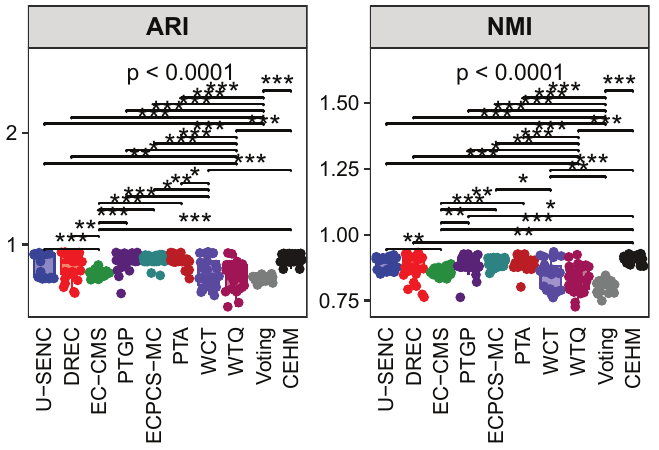}}\;
\caption{The Box Plot with Statistical Significance on data set 7}
\label{pdata7}
\end{figure}

\begin{figure}[!ht]
\captionsetup[subfigure]{labelformat=empty}
{\includegraphics[width=0.45\textwidth,height=0.2\textheight]{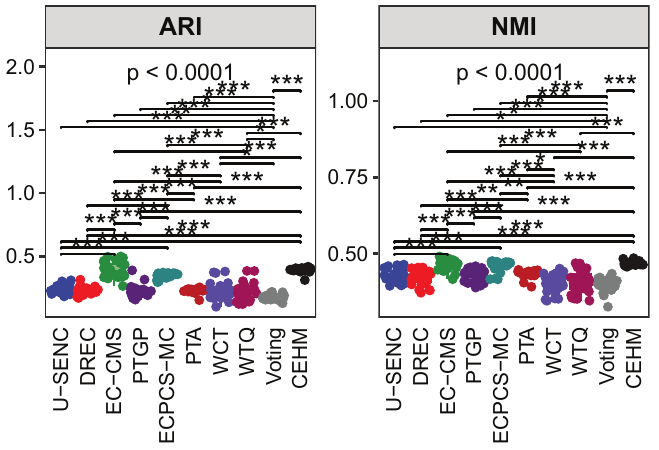}}\;
\caption{The Box Plot with Statistical Significance on data set 8}
\label{pdata8}
\end{figure}

\begin{figure}[!ht]
\captionsetup[subfigure]{labelformat=empty}
{\includegraphics[width=0.45\textwidth,height=0.2\textheight]{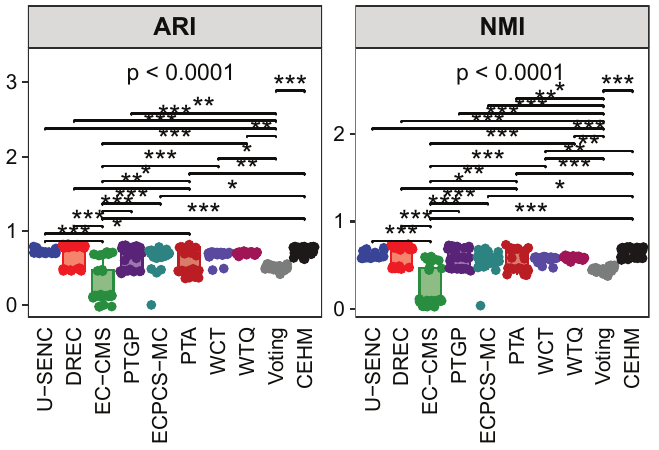}}\;
\caption{The Box Plot with Statistical Significance on data set 9}
\label{pdata9}
\end{figure}

\begin{figure}[!ht]
\captionsetup[subfigure]{labelformat=empty}
{\includegraphics[width=0.45\textwidth,height=0.2\textheight]{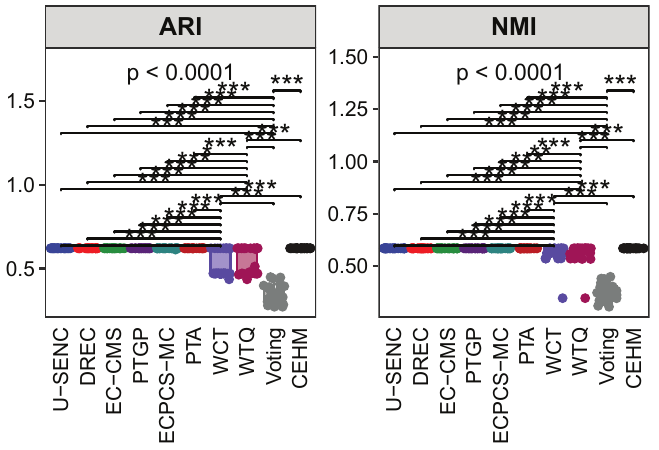}}\;
\caption{The Box Plot with Statistical Significance on data set 10}
\label{pdata10}
\end{figure}

\begin{figure}[!ht]
\captionsetup[subfigure]{labelformat=empty}
{\includegraphics[width=0.45\textwidth,height=0.2\textheight]{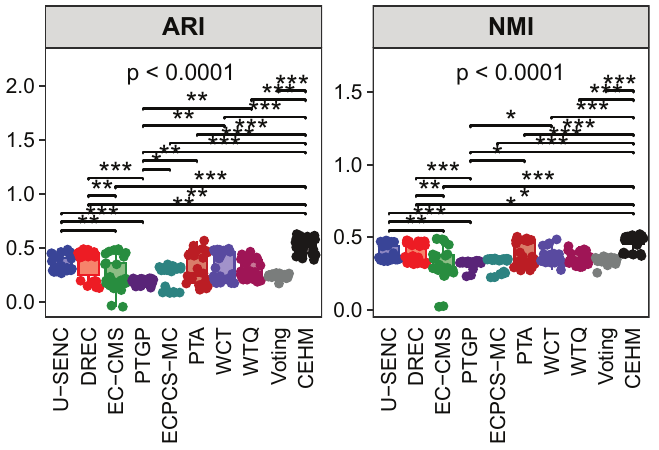}}\;
\caption{The Box Plot with Statistical Significance on data set 11}
\label{pdata11}
\end{figure}

\begin{figure}[!ht]
\captionsetup[subfigure]{labelformat=empty}
{\includegraphics[width=0.45\textwidth,height=0.2\textheight]{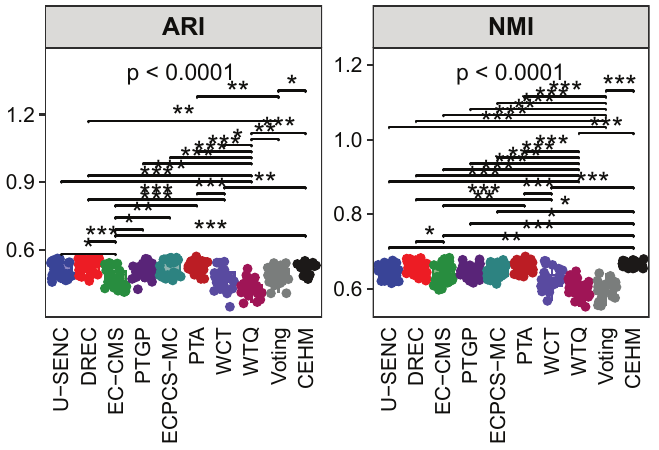}}\;
\caption{The Box Plot with Statistical Significance on data set 12}
\label{pdata12}
\end{figure}

\begin{figure}[!ht]
\captionsetup[subfigure]{labelformat=empty}
{\includegraphics[width=0.45\textwidth,height=0.2\textheight]{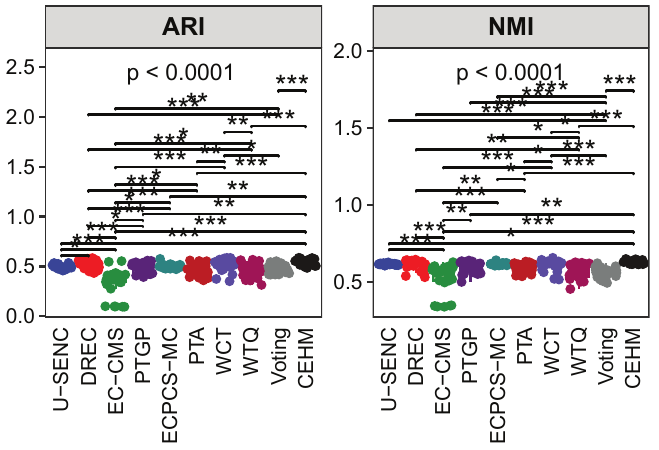}}\;
\caption{The Box Plot with Statistical Significance on data set 13}
\label{pdata13}
\end{figure}

\begin{figure}[!ht]
\captionsetup[subfigure]{labelformat=empty}
{\includegraphics[width=0.45\textwidth,height=0.2\textheight]{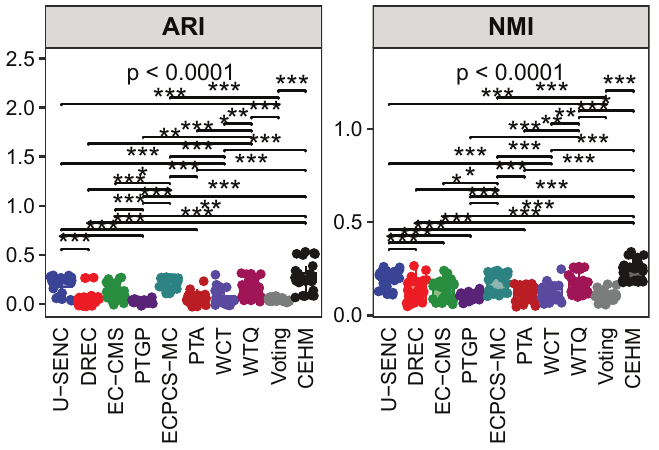}}\;
\caption{The Box Plot with Statistical Significance on data set 14}
\label{pdata14}
\end{figure}

\begin{figure}[!ht]
\captionsetup[subfigure]{labelformat=empty}
{\includegraphics[width=0.45\textwidth,height=0.2\textheight]{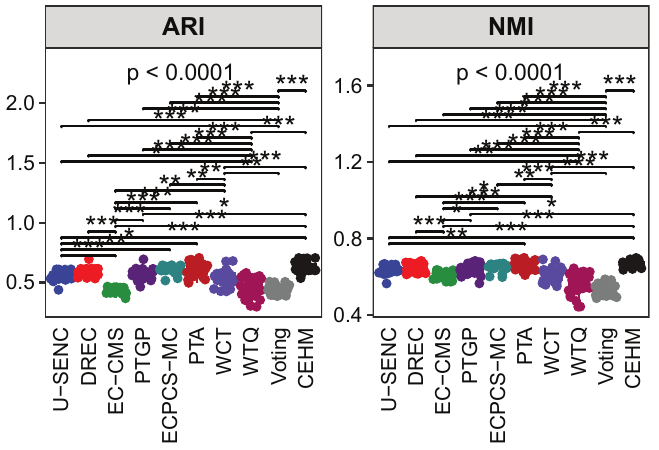}}\;
\caption{The Box Plot with Statistical Significance on data set 15}
\label{pdata15}
\end{figure}

\begin{figure}[!ht]
\captionsetup[subfigure]{labelformat=empty}
{\includegraphics[width=0.45\textwidth,height=0.2\textheight]{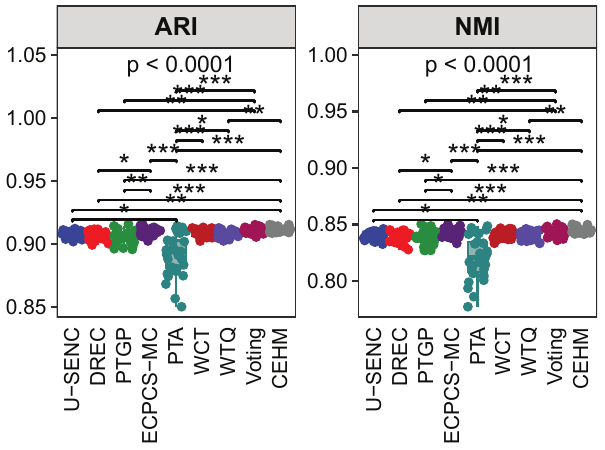}}\;
\caption{The Box Plot with Statistical Significance on data set 16}
\label{pdata16}
\end{figure}

\begin{figure}[!ht]
\captionsetup[subfigure]{labelformat=empty}
{\includegraphics[width=0.45\textwidth,height=0.2\textheight]{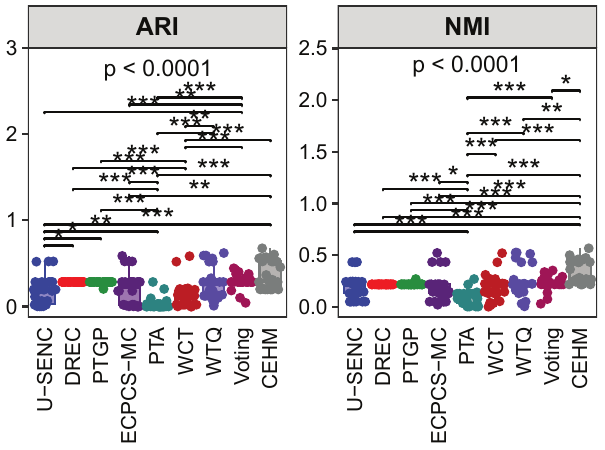}}\;
\caption{The Box Plot with Statistical Significance on data set 17}
\label{pdata17}
\end{figure}

\begin{figure}[!h]
\captionsetup[subfigure]{labelformat=empty}
{\includegraphics[width=0.45\textwidth,height=0.2\textheight]{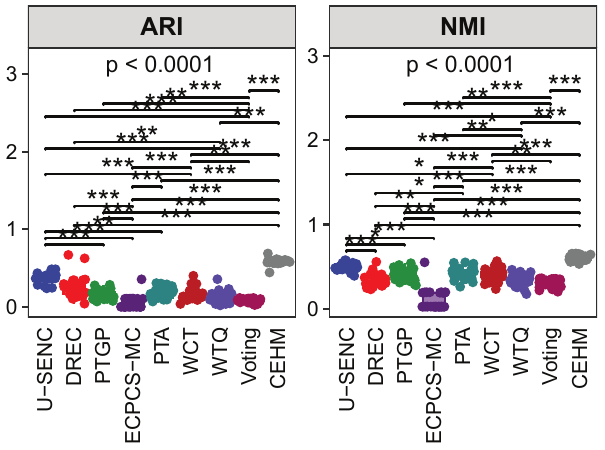}}\;
\caption{The Box Plot with Statistical Significance on data set 18}
\label{pdata18}
\end{figure}

\begin{figure}[!h]
\includegraphics[width=0.45\textwidth,height=0.2\textheight]{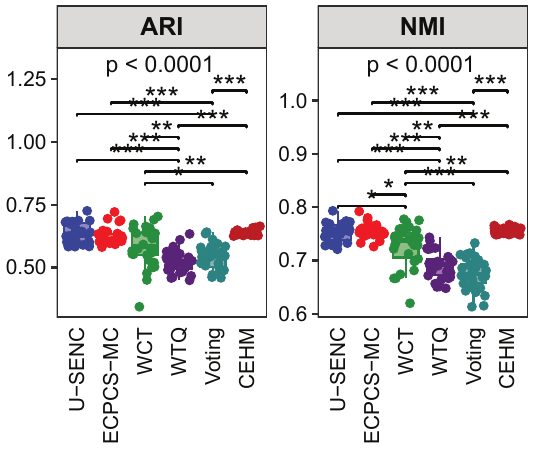}
\caption{The Box Plot with Statistical Significance on data set 19}
\label{pdata19}
\end{figure}

\begin{figure}[!h]
\includegraphics[width=0.45\textwidth,height=0.2\textheight]{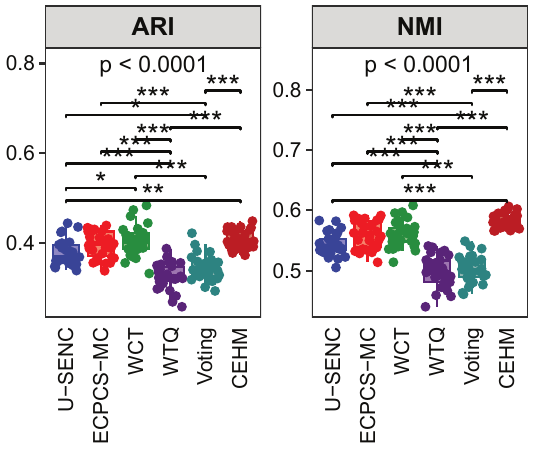}
\caption{The Box Plot with Statistical Significance on data set 20}
\label{pdata20}
\end{figure}

\end{document}